\documentclass{article}

    \PassOptionsToPackage{numbers, compress}{natbib}

\usepackage[preprint]{neurips_2026}

\usepackage[utf8]{inputenc} 
\usepackage[T1]{fontenc}    
\usepackage{hyperref}       
\usepackage{url}            
\usepackage{booktabs}       
\usepackage{amsfonts}       
\usepackage{nicefrac}       
\usepackage{microtype}      
\usepackage{xcolor}         

\usepackage{amsmath}
\usepackage{amssymb}
\usepackage{mathtools}
\usepackage{amsthm}

\usepackage{subcaption}

\usepackage{tikz}

\usepackage{algorithm,algorithmicx,algpseudocode}

\usepackage{multirow}

\usepackage{color,soul}

\usepackage{wrapfig}

\newcommand{\mstd}[2]{#1$\pm${\scriptsize #2}}
\newcommand{\bmstd}[2]{\textbf{#1$\pm${\scriptsize #2}}}

\usepackage{enumitem}
\setlist[itemize]{leftmargin=*}

\usepackage[capitalize,noabbrev]{cleveref}

\theoremstyle{plain}
\newtheorem{theorem}{Theorem}[section]
\newtheorem{proposition}[theorem]{Proposition}

\theoremstyle{definition}

\newtheorem{assumption}[theorem]{Assumption}
\theoremstyle{remark}

\usepackage[disable,textsize=tiny]{todonotes}

\usepackage{url}

\title{Rethinking Molecular Graph Backdoors under Chemistry-aware Admission}

\author{%
  \textbf{Thinh T. H. Nguyen$^{1}$, \quad
  Sze Jue Yang$^{2}$, \quad
  Khoa D. Doan$^{1}$,}\\
  \textbf{Chee Seng Chan$^{1,2}$, \quad
  Kok-Seng Wong$^{1}$}\thanks{Corresponding author: \url{wong.ks@vinuni.edu.vn}}\\
  $^{1}$VinUniversity, Hanoi, Vietnam, \quad
  $^{2}$Universiti Malaya\\
  \texttt{thinh.nth@vinuni.edu.vn, jasonyang429@um.edu.my, khoa.dd@vinuni.edu.vn,}\\
  \texttt{cs.chan@um.edu.my, wong.ks@vinuni.edu.vn}
}

\begin{document}

\maketitle

\begin{abstract}
Backdoor attacks on molecular graph neural networks (GNNs) are typically evaluated as abstract graph edits, but real molecular learning pipelines do not train on arbitrary graphs. Molecular records must first survive parsing, sanitization, canonicalization, and graph-string consistency checks. 
We formalize this overlooked admission stage as \textbf{ChemGuard}, an operational protocol for testing whether a submitted molecular record can enter a realistic learning pipeline, while complementing existing defenses.
\textbf{ChemGuard} admits a record only when its molecular string is sanitizable and the graph reconstructed from that string matches the submitted molecular graph. Under this operational view, many existing graph-based backdoors lose much of their apparent efficacy because their poisons are chemically invalid or representation-inconsistent. We then show that admission checks alone are insufficient to rule out molecular backdoors. We propose \textbf{ChemBack}, an admission-aware molecular backdoor attack that constructs chemically feasible motif-anchor attachments and ranks admitted candidates by fingerprint-based Tanimoto similarity to clean target-class molecules. \textbf{ChemBack} is model-free during trigger selection, using molecular structures, target labels, fingerprints, and public validity checks, but no victim model, surrogate GNN, learned embedding, gradient, logit, or training-code access. Across molecular benchmarks, validators, architectures, and defenses, \textbf{ChemBack} achieves high attack success with fully admitted poisons while preserving clean accuracy. Our results reveal a two-sided lesson, \textit{chemistry-aware admission suppresses many graph-only backdoors, yet chemically valid and target-aligned molecular backdoors remain a practical threat}.
\end{abstract}

\section{Introduction}
\vspace{-3pt}

Molecular Graph Neural Networks (GNNs) are widely used for molecular property prediction, drug discovery, and materials science~\cite{li2018deeper,zhou2020graph,wu2018moleculenet,gilmer2017neural}. Recent studies~\cite{zheng2023motif,dai2023unnoticeable,zhang2024rethinking} show that molecular GNNs are vulnerable to backdoor attacks, where an adversary poisons a small fraction of training data with a trigger pattern and assigns the poisoned samples to a target label. At inference time, clean molecules are classified normally, while triggered molecules are pushed toward the attacker-specified label.
Such attacks raise security concerns for molecular prediction systems in chemical-related applications, where manipulated predictions may lead to unsafe downstream decisions~\cite{kang2022lr,long2022pre,wieder2020compact,satheeskumar2025enhancing}.

However, many molecular backdoor evaluations still treat trigger insertion as an abstract graph edit, assuming a poisoned molecule becomes available once the trigger is added. This assumption does not match realistic cheminformatics workflows, where records are parsed, sanitized, canonicalized, and featurized before training or inference~\cite{landrum2013rdkit,o2011open,pavlov2011indigo}. These steps enforce valence validity, aromaticity normalization, bond consistency, and graph-string consistency. Thus, an abstractly valid trigger may fail to become an admissible molecular record. Ignoring this admission stage can inflate measured attack success by counting invalid poisons as if they reached the learner, as shown in Figure~\ref{fig:teaser}.

\begin{wrapfigure}{l}{0.65\linewidth}
\vspace{-10pt}
\centering
\includegraphics[width=\linewidth, keepaspectratio=true]{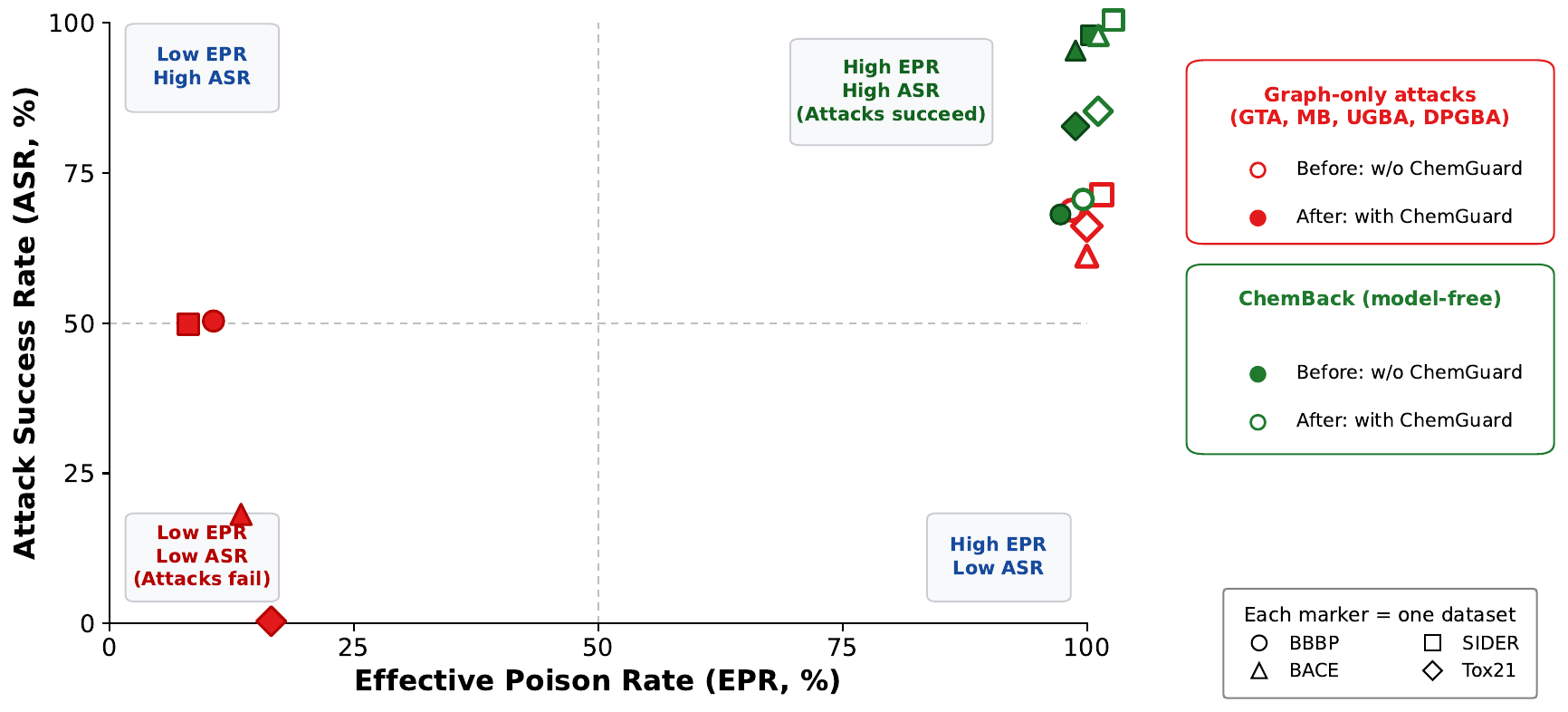}
\caption{EPR-ASR summary of graph-only attacks and \textbf{ChemBack} before and after \textbf{ChemGuard}.
\textbf{ChemGuard} shifts graph-only attacks from the high-EPR/high-ASR region toward lower-EPR/lower-ASR regions because many attempted poisons or test-time triggers fail chemistry-aware admission. In contrast, \textbf{ChemBack} remains in the high-EPR/high-ASR region, indicating that it remains operational after chemistry-aware admission.}
\label{fig:teaser}
\vspace{-10pt}
\end{wrapfigure}

To model this gap, we introduce \textbf{ChemGuard}, a chemistry-aware admission protocol. 
\textbf{ChemGuard} is an admission protocol that complements existing defenses, where it examines whether a submitted molecular record can enter a realistic molecular learning pipeline.
A record is admitted only when its molecular string is sanitizable and the graph reconstructed from that string is consistent with the submitted graph. Under this protocol, representative graph backdoors lose much of their operational attack signal because many attempted poisons or test-time triggers fail admission.

This reframes molecular backdoor evaluation from \emph{whether an abstract graph trigger can fool a GNN} to \emph{whether a poisoned molecular record can pass chemistry-aware admission and still induce targeted behavior}. 
Hence, we ask:
\textbf{\emph{Can effective molecular backdoors survive realistic chemistry-aware admission before training and inference?}}

We answer this question with \textbf{ChemBack}, an admission-aware molecular backdoor attack that attaches chemically feasible motifs at valid atomic anchors, keeping poisoned molecules sanitizable and graph-string consistent under \textbf{ChemGuard}. Since chemical validity alone does not ensure effective or stealthy poisoning, \textbf{ChemBack} selects triggers that are realizable, admitted by preprocessing, reusable at training and test time, and close to the target class. Unlike UGBA~\cite{dai2023unnoticeable} and DPGBA~\cite{zhang2024rethinking}, which rely on distribution preservation or surrogate-model optimization, \textbf{ChemBack} requires no victim model or surrogate GNN during trigger selection. Instead, it searches feasible motif-anchor attachments and ranks \textbf{ChemGuard}-admissible candidates by Tanimoto similarity~\cite{rogers2010extended,bajusz2015tanimoto} to clean target-class molecules, using only molecular structures, target labels, and public fingerprint functions. This makes \textbf{ChemBack} more than a motif backdoor with a validity filter, as triggers must remain admissible, graph-string consistent, target-aligned, and model-free during both poisoning and activation. Clean reference models are used only for post-hoc diagnostics, not attack construction.

We validate these claims through a broad empirical study. We first show that chemistry-aware admission substantially reduces the operational attack signal of representative graph backdoors.
We then show that \textbf{ChemBack} remains effective under the same admission rule, achieving high attack success while maintaining clean accuracy. This pattern persists across molecular benchmarks, chemistry validators, GNN architectures, larger-scale datasets, pretrain-finetune settings, and representative defenses, revealing an attack vector that remains operational after chemistry-aware admission.

In summary, our contributions are:
\begin{enumerate}[leftmargin=*]
\item We identify an \emph{admission gap}, showing that graph-level trigger success can overestimate operational risk when sanitization and graph-string consistency are ignored.
\item We formalize \textbf{ChemGuard}, a minimal pipeline-level admission protocol for testing whether a molecular record can enter training or inference after chemistry-aware preprocessing.
\item We propose \textbf{ChemBack}, an admission-aware molecular backdoor attack that builds chemically feasible motif-anchor attachments and selects admitted candidates by Tanimoto similarity to clean target-class molecules, without victim or surrogate model access.
\item We show across benchmarks, larger datasets, pretrain-finetune settings, validators, architectures, and defenses that chemistry-aware admission weakens graph-only backdoors, while chemically valid and target-aligned backdoors remain effective.
\end{enumerate}

\section{Related Work}
\label{sec:related}

\textbf{Graph Neural Networks (GNNs).}
GNNs learn latent representations by aggregating information over graph structure~\cite{scarselli2009neural}. Molecules are naturally represented as graphs, with atoms as nodes and chemical bonds as edges, making GNNs effective for molecular property prediction and related chemistry tasks~\cite{duvenaud2015convolutional,wu2018moleculenet,gilmer2017neural}. Architectures such as GCN~\cite{kipf2017semisupervised} and GraphSAGE~\cite{hamilton2017inductive} are widely used in molecular learning, while large-scale self-supervised models such as GROVER~\cite{rong2020self} further improve molecular representations through pretraining. As molecular GNNs become more common in scientific pipelines, their security surface also expands, especially when training data are collected from public or third-party sources.

\textbf{Backdoor attacks (BA).}
BA poison training data so that a model behaves normally on clean inputs but predicts an attacker-specified target label when a trigger appears. BadNets~\cite{gu2017badnets} introduced this threat in images, followed by works on diverse triggers, budgets, and stealth objectives~\cite{nguyen2020input,doan2021lira,nguyen2024backdoor}.
In graphs, GTA~\cite{xi2021graph} introduced subgraph triggers, while UGBA~\cite{dai2023unnoticeable} and DPGBA~\cite{zhang2024rethinking} improve unnoticeability and distribution alignment. For molecules, Motif-Backdoor~\cite{zheng2023motif} uses recurring substructures as triggers. These methods reveal important GNN vulnerabilities, but they often treat triggers as abstract graph edits. In realistic cheminformatics pipelines, edited records must still pass sanitization, canonicalization, and graph-string consistency checks before entering training or inference.

\textbf{Backdoor and poisoning defenses.}
Existing defenses operate at different stages of the learning pipeline. Backdoor-specific defenses such as Neural Cleanse~\cite{wang2019neural}, Spectral Signatures~\cite{tran2018spectral}, DShield~\cite{yu2025dshield}, and RIGBD~\cite{zhang2024robustness} typically act after samples have entered training, using trigger reconstruction, embedding outliers, discrepancy, commonality, or robustness-based signals. Robust GNN and graph-poisoning defenses, including RGCN~\cite{zhu2019robust}, GNNGuard~\cite{zhang2020gnnguard}, Pro-GNN~\cite{jin2020graph}, and certified defenses~\cite{wang2021certified,yang2024distributed,li2025deterministic}, also operate after graph admission. By contrast, \textbf{ChemGuard} targets admission by testing whether a molecular record survives sanitization and graph-string consistency checks before entering the learner. Thus, chemistry-aware admission is complementary to post-hoc and model-level defenses, not a replacement for them.

\section{Methodology}
\label{sec:method}
\vspace{-3pt}

\subsection{Preliminaries and Threat Model}
\label{sec:prelim-threat}

\textbf{Molecular records and preprocessing.}
We represent a molecular record as $(\tilde{s},\tilde{G},y)$, where $\tilde{s}\in\mathcal{S}$ is a molecular string, $\tilde{G}$ is the submitted or stored molecular graph, and $y\in\mathcal{Y}$ is the task label. A molecule is an attributed graph $G=(V,E,X,B)$, where $V$ and $E$ denote atoms and bonds, $X\in\mathbb{R}^{|V|\times d_v}$ is the atom-feature matrix, and $B\in\mathbb{R}^{|E|\times d_e}$ is the bond-feature matrix. Molecular strings are commonly parsed, sanitized, canonicalized, and featurized by cheminformatics toolkits~\cite{landrum2013rdkit,o2011open,pavlov2011indigo}.
We model this preprocessing stack as a deterministic operator $\phi_{\mathcal{T}}:\mathcal{S}\rightarrow \mathcal{G}_{\mathrm{valid}}\cup\{\varnothing\},$
where $\mathcal{T}$ is the chemistry toolkit or validation stack, and $\mathcal{G}_{\mathrm{valid}}$ denotes the set of chemically valid molecular graphs. If parsing, sanitization, canonicalization, and featurization succeed, $\phi_{\mathcal{T}}(\tilde{s})$ returns a valid molecular graph; otherwise it returns $\varnothing$.
A molecular GNN is a classifier $f_\theta:\mathcal{G}_{\mathrm{valid}}\rightarrow[0,1]^M$, where $M$ is the number of classes. The effective deployed prediction pipeline is therefore $f_\theta\circ\phi_{\mathcal{T}}$.

\textbf{Learning objective.}
Let $D_{\mathrm{adm}}=\{(\tilde{s}_i,\tilde{G}_i,y_i)\}_{i=1}^{N}$ denote the training records admitted by the preprocessing pipeline.
The victim trains $f_\theta$ by empirical risk minimization:
\begin{equation}
\min_\theta
\frac{1}{|D_{\mathrm{adm}}|}
\sum_{(\tilde{s}_i,\tilde{G}_i,y_i)\in D_{\mathrm{adm}}}
\ell\!\left(f_\theta(\phi_{\mathcal{T}}(\tilde{s}_i)),y_i\right).
\label{eq:erm}
\end{equation}
The stored graph $\tilde{G}_i$ is included because admission may check consistency between the submitted graph and the graph reconstructed from the molecular string. Prediction is computed from the admitted molecular string via $\phi_{\mathcal{T}}$.
This formulation emphasizes that the learner trains only on records that survive molecular preprocessing, not arbitrary raw graph edits.

\textbf{Threat model.}
We consider a data-poisoning adversary who can inject or modify a small fraction $\alpha$ of raw molecular records before admission.
The attacker fixes a target label $y_t$ and aims to preserve clean performance while causing triggered non-target molecules to be classified as $y_t$. The attacker has no access to the victim model or private training pipeline and cannot access or modify the victim parameters $\theta$, learned representations, gradients, optimizer state, training code, random seed, training trajectory, or private pre-processing configuration. 
Also, the attacker is refrained from forcing molecules that are invalid or possess graph-string inconsistencies to bypass pre-processing.


\textbf{Backdoor poisoning.}
Let $\tau=(\tau_s,\tau_G)$ denote a trigger transformation, where $\tau_s:\mathcal{S}\rightarrow\mathcal{S}$ modifies the molecular string and $\tau_G$ denotes the corresponding submitted-graph transformation.
For a clean non-target molecule $(\tilde{s}_c,\tilde{G}_c,y_c)$ with $y_c\neq y_t$, a targeted poisoning attempt produces
\[
(\tilde{s}_p,\tilde{G}_p,y_t)
=
\big(\tau_s(\tilde{s}_c),\tau_G(\tilde{G}_c),y_t\big)
\in D_{\mathrm{poison}}^{\mathrm{att}}.
\]
The nominal poison rate $\alpha$ controls how many raw training records the attacker attempts to poison.
However, in chemistry-aware pipelines, not every attempted poison reaches the learner: only records admitted by the preprocessing pipeline can contribute to training.

\textbf{Prediction under an admission rule.}
Let $\mathrm{Adm}(\tilde{s},\tilde{G})\in\{0,1\}$ denote an admission rule. For convenience, we define the admitted prediction as
\[
\hat{y}_\theta(\tilde{s},\tilde{G})
=
\begin{cases}
\arg\max_{c\in\{1,\ldots,M\}} f_\theta(\phi_{\mathcal{T}}(\tilde{s}))_c, 
& \mathrm{Adm}(\tilde{s},\tilde{G})=1,\\
\bot, 
& \mathrm{Adm}(\tilde{s},\tilde{G})=0.
\end{cases}
\]
Here, $\bot$ denotes rejection before model evaluation.
In raw abstract-graph evaluation, $\mathrm{Adm}$ is treated as always one. Under chemistry-aware evaluation, $\mathrm{Adm}$ is instantiated by \textbf{ChemGuard}.

\textbf{Metrics.}
We report Clean Accuracy (CA), operational Attack Success Rate (ASR), and Effective Poisoning Rate (EPR).
Let $D_{\mathrm{clean}}^{\mathrm{adm}}
=
\{(\tilde{s}_c,\tilde{G}_c,y_c)\in D_{\mathrm{clean}}:\mathrm{Adm}(\tilde{s}_c,\tilde{G}_c)=1\}$; CA measures utility on unmodified admitted test molecules, where
$\mathrm{CA}
=
\frac{1}{|D_{\mathrm{clean}}^{\mathrm{adm}}|}
\sum_{(\tilde{s}_c,\tilde{G}_c,y_c)\in D_{\mathrm{clean}}^{\mathrm{adm}}}
\mathbf{1}\!\left[
\hat{y}_{\theta}(\tilde{s}_c,\tilde{G}_c)=y_c
\right].$
Let 
$D_{\mathrm{poison}}^\mathrm{att}
=
\{(\tilde{s}_p,\tilde{G}_p,y_t)\in D_{\mathrm{poison}}:y\neq y_t\}$. Operational ASR measures targeted behavior on triggered non-target test molecules, where
$
\mathrm{ASR}
=
\frac{1}{|D_{\mathrm{test}}^{\neg y_t}|}
\sum_{(\tilde{s},\tilde{G},y)\in D_{\mathrm{test}}^{\neg y_t}}
\mathbf{1}\!\left[
\hat{y}_{\theta}(\tau_s(\tilde{s}),\tau_G(\tilde{G}))=y_t
\right]$.
If a triggered test molecule fails admission, then
$\hat{y}_{\theta}(\tilde{s}_p, \tilde{G}_p)=\bot$, and the indicator contributes zero. Thus, triggered molecules that cannot be realized as admissible molecular records are counted as attack failures.

\subsection{\textbf{ChemGuard}}
\label{sec:chemguard}

Molecular records are not arbitrary graphs.
Before a submitted record can be used by a practical molecular learning pipeline, it must satisfy chemical and representational constraints.
Structural modifications must respect atom identities, valence, bond types, aromaticity normalization, and consistency between the molecular string and the submitted graph.
Existing graph backdoor evaluations often do not model this admission stage and may count an abstract graph trigger as successful even when the corresponding molecular record would fail sanitization or graph--string consistency.

\textbf{Definition 1 (\textbf{ChemGuard} admission).}
Let $(\tilde{s},\tilde{G},y)$ be a submitted molecular record, and let $\phi_{\mathcal{T}}:\mathcal{S}\rightarrow\mathcal{G}_{\mathrm{valid}}\cup\{\varnothing\}$ denote the preprocessing operator induced by chemistry toolkit $\mathcal{T}$.
\textbf{ChemGuard} admits $(\tilde{s},\tilde{G},y)$ iff:
(i) $\phi_{\mathcal{T}}(\tilde{s})\neq\varnothing$; and
(ii) the typed molecular topology reconstructed from $\tilde{s}$ matches the submitted graph $\tilde{G}$.
Formally,
\begin{equation}
\mathrm{ChemGuard}_{\mathcal{T}}(\tilde{s},\tilde{G})
=
\mathbf{1}\!\left[\phi_{\mathcal{T}}(\tilde{s})\neq\varnothing\right]
\cdot
\mathbf{1}\!\left[
\mathrm{Topo}\!\left(\phi_{\mathcal{T}}(\tilde{s})\right)
=
\mathrm{Topo}(\tilde{G})
\right].
\label{eq:chemguard}
\end{equation}
Here, $\mathrm{Topo}(\tilde{G})$ denotes typed molecular topology, including atom identities, atom count, undirected bond set, and bond types under the toolkit-induced canonical representation.

\textbf{ChemGuard} is intentionally minimal.
It is not a complete defense because it does not inspect learned representations, detect triggers, or certify robustness against arbitrary graph perturbations.
Instead, it formalizes the implicit admission rule that molecular records must be chemically valid and representation-consistent before training or inference.
Thus, \textbf{ChemGuard} complements post-hoc backdoor and model-level poisoning defenses, which operate after sample admission.

Because \textbf{ChemGuard} reject many attempted poisons, the nominal poison rate $\alpha$ may differ from the poisoned signal that reaches training. We define the \textbf{ChemGuard} Effective Poisoning Rate as
\begin{equation}
\mathrm{EPR}
=
\frac{1}{|D_{\mathrm{poison}}^{\mathrm{att}}|}
\sum_{(\tilde{s}_p,\tilde{G}_p,y_t)\in D_{\mathrm{poison}}^{\mathrm{att}}}
\mathrm{ChemGuard}_{\mathcal{T}}(\tilde{s}_p,\tilde{G}_p).
\label{eq:epr-cg}
\end{equation}
Here, $D*{\mathrm{poison}}^{\mathrm{att}}$ is the set of poisoned records attempted before admission. The nominal poison rate $\alpha$ is the attacker’s attempted budget, while $\mathrm{EPR}*{\mathrm{CG}}$ is the fraction of attempted poisons admitted into training. In raw abstract-graph evaluation, no chemistry-aware gate is enforced, so EPR is reported only for the \textbf{ChemGuard} mode. The same rule is applied at test time, where triggered molecules that fail \textbf{ChemGuard} are rejected before evaluation and counted as ASR failures.

\subsection{\textbf{ChemBack}}
\label{sec:chemback}

\textbf{ChemBack} tests whether effective molecular backdoors remain possible after \textbf{ChemGuard} removes invalid or graph-string inconsistent poisons. It imposes two requirements. Poisoned molecules must be \emph{chemically admissible}, meaning sanitizable and graph-string consistent. Admitted poisons should also be \emph{target-aligned}, meaning structurally close to clean target-class molecules under chemistry-facing fingerprints. This alignment helps the trigger induce the target label while remaining less outlying under the diagnostics considered in this work. Moreover, \textbf{ChemBack} is model-free during trigger selection, using only molecular structures, target labels, public chemistry-validity checks, fingerprints, and Tanimoto similarity. Clean reference models are used only post-hoc for diagnostics.

\textbf{Definition 2 (\textbf{ChemBack} attack).}
Let $D_{\mathrm{avail}}$ be the molecular records available to the attacker before admission, let $y_t\in\mathcal{Y}$ be the attacker-specified target label, and let $\mathrm{ChemGuard}_{\mathcal{T}}$ be the admission protocol induced by toolkit $\mathcal{T}$.
\textbf{ChemBack} is an admission-aware poisoning transformation $\tau=(\tau_s,\tau_G)$ that maps a clean non-target molecule $(\tilde{s},\tilde{G},y)$ with $y\neq y_t$ to a poisoned record
\[
(\tilde{s}_p,\tilde{G}_p,y_t)
=
\big(\tau_s(\tilde{s}),\tau_G(\tilde{G}),y_t\big),
\qquad
\mathrm{ChemGuard}_{\mathcal{T}}(\tilde{s}_p,\tilde{G}_p)=1.
\]
The trigger rule $\tau$ is constructed from \textbf{ChemGuard}-admissible motif attachments selected by fingerprint-based Tanimoto similarity to clean target-class molecules.
At test time, the same $\tau$ is applied to non-target molecules, and any triggered record that fails \textbf{ChemGuard} is counted as an attack failure under operational ASR.

\textbf{Design overview.}
To instantiate $\tau$, \textbf{ChemBack} mines candidate molecular motifs, enumerates feasible host and motif anchors, and attempts single-bond motif attachments. Each candidate is sanitized, canonicalized, reconstructed as a molecular graph, and checked by \textbf{ChemGuard}. Among admitted candidates, \textbf{ChemBack} selects the trigger using fingerprint-based Tanimoto similarity to clean target-class molecules. Unlike representation-guided graph backdoors, this score is computed from deterministic molecular fingerprints. Figure~\ref{fig:chemback} shows the pipeline.

\begin{figure}[t]
\centering
\includegraphics[width=0.9\linewidth,keepaspectratio=true]{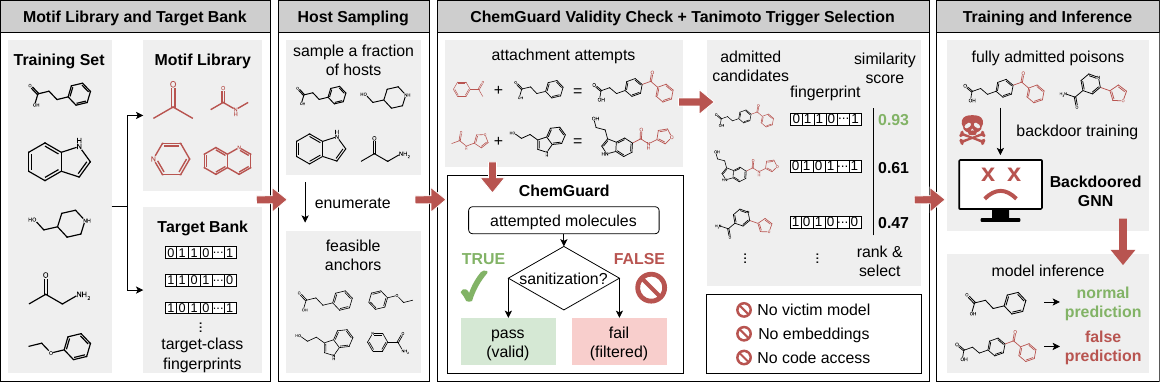}
\caption{\textbf{Overview of \textbf{ChemBack} under \textbf{ChemGuard}.} \textbf{ChemBack} forms a trigger library from candidate motifs, attaches them to sampled non-target hosts, and filters feasible motif-anchor attachments with \textbf{ChemGuard} for sanitization and graph-string consistency. It then selects admitted triggers by fingerprint-based Tanimoto similarity to clean target-class molecules. The selected trigger produces \textbf{ChemGuard}-admissible training poisons and is reused for test-time realization. Selection uses no victim model or surrogate GNN.}
\label{fig:chemback}
\vspace{-10pt}
\end{figure}

\textbf{Chemically constrained trigger space.}
\textbf{ChemBack} constructs a motif library $\mathcal{M}=\{m_k\}_{k=1}^{K}$, where each motif $m_k$ is a chemically meaningful substructure mined from available molecular data. Motifs may be selected to be rare under the training distribution to reduce accidental trigger occurrence, but rarity is not treated as a stealth guarantee. The final trigger is selected by chemical admissibility under \textbf{ChemGuard} and target-class structural alignment under molecular fingerprints.

Given a host molecule $G=(V,E)$, \textbf{ChemBack} restricts host anchors to atoms that can support an additional chemically valid bond:
\[
\mathcal{A}(G)=\{i\in V:\text{atom }i\text{ admits an additional chemically valid bond}\}.
\]
For each motif $m_k$, it also enumerates feasible motif-side anchors.
An action is $a=(k,i,j)$, where $k$ selects the motif, $i$ the host anchor, and $j$ the motif anchor. The attachment operator $\mathrm{Att}_{\mathcal{T}}(G,m_k,i,j)$ connects the host and motif through a valid bond, sanitizes the result, canonicalizes its molecular string, and reconstructs the corresponding graph.
If the edit violates chemical constraints or fails \textbf{ChemGuard}, the operator returns $\bot$; otherwise, it returns $(\tilde{s}_p, \tilde{G}_p)$ with $\mathrm{ChemGuard}_{\mathcal{T}}(\tilde{s}_p,\tilde{G}_p)=1$.

\textbf{Target fingerprint bank.}
To select target-aligned triggers without a victim or surrogate model, \textbf{ChemBack} constructs
\begin{equation}
B_t =
\left\{
\mathrm{FP}\!\left(\phi_{\mathcal{T}}(\tilde{s}_i)\right):
(\tilde{s}_i,\tilde{G}_i,y_i)\in D_{\mathrm{avail}},
\ y_i=y_t,
\ \mathrm{ChemGuard}_{\mathcal{T}}(\tilde{s}_i,\tilde{G}_i)=1
\right\}.
\label{eq:target-bank}
\end{equation}
Here, $\mathrm{FP}:\mathcal{G}_{\mathrm{valid}}\rightarrow\{0,1\}^d$ is a molecular fingerprint function. Unless otherwise stated, we use Morgan fingerprints with radius 2 and 2048 bits~\cite{rogers2010extended}.
For binary fingerprints $u,v\in\{0,1\}^d$, their Tanimoto similarity is $\mathrm{Tan}(u,v)= \frac{|u\wedge v|}{|u\vee v|}.$ This provides a deterministic, model-free estimate of target-class structural similarity~\cite{bajusz2015tanimoto}.

\textbf{Tanimoto-guided trigger selection.}
For each admitted candidate $G_{\mathrm{poi}}$, \textbf{ChemBack} computes $S_{\mathrm{Tan}}(\tilde{G}_p)=\max_{z\in B_t}\mathrm{Tan}(\mathrm{FP}(\tilde{G}_p),z)$. The reward for action $a=(k,i,j)$ on host $G$ is
\begin{equation}
r(G,k,i,j)
=
\begin{cases}
r_{\mathrm{inv}}, & \mathrm{Att}_{\mathcal{T}}(G,m_k,i,j)=\bot,\\[1mm]
1+\lambda_{\mathrm{Tan}}S_{\mathrm{Tan}}(\tilde{G}_p), & \text{otherwise},
\end{cases}
\label{eq:tan-reward}
\end{equation}
where $r_{\mathrm{inv}}<0$ penalizes invalid or inadmissible attachments, and $\lambda_{\mathrm{Tan}}$ controls target-class structural alignment. The constant term rewards admission, while the Tanimoto term prefers admitted molecules structurally similar to clean target-class molecules.

\textbf{Trigger freezing and poisoning.}
A backdoor trigger should be shared across poisoned training samples and test-time triggered inputs. \textbf{ChemBack} therefore searches motif-anchor candidates on a calibration subset of non-target hosts and selects a trigger motif or small trigger family with high admission success and high target-class Tanimoto similarity. It poisons a fraction $\alpha$ of non-target training molecules by applying the selected trigger, retaining only records satisfying $\mathrm{ChemGuard}_{\mathcal{T}}(\tilde{s}_p,\tilde{G}_p)=1$, and assigning label $y_t$. At test time, the same trigger rule is applied to non-target test molecules. If the trigger cannot be realized as a \textbf{ChemGuard}-admissible molecular record, that attempt is counted as a failure under operational ASR.

\textbf{Optimization.}
\textbf{ChemBack} performs black-box search over a discrete space of chemistry-admissible motif attachments, including motif choice, host-anchor choice, motif-anchor choice, sanitization, canonicalization, and graph-string consistency checks. This search can be implemented using deterministic, greedy, or reinforcement-learning strategies.
The method is optimizer-agnostic and depends only on \textbf{ChemGuard}-admissible construction and Tanimoto-guided target alignment.

\section{Experiments}
\label{sec:experiments}


\subsection{Experimental Setup}
\label{sec:main-setup}

\textbf{Datasets.}
We evaluate on four MoleculeNet benchmarks, including \textbf{BBBP}~\cite{martins2012bayesian}, \textbf{BACE}~\cite{subramanian2016computational}, \textbf{SIDER}~\cite{kuhn2016sider}, and \textbf{Tox21}~\cite{huang2016tox21challenge}.
BBBP and BACE are single-task binary classification datasets.
SIDER and Tox21 are multi-task benchmarks with 27 and 12 tasks, respectively. For SIDER and Tox21, we run all tasks and report macro-averaged results in the main text. Task-wise results are deferred to Appendix~\ref{app:taskwise-asr-chemback}. Additional tests on larger benchmarks, pretrain-finetune pipelines, alternative chemistry validators, and additional structural diagnostics are reported in Appendix~\ref{app:scale-pretrain} and Appendix~\ref{app:structural-diagnostics}.

\textbf{Hyperparameters.}
Following MoleculeNet split policies, we use scaffold splits for BBBP and BACE, and random splits for SIDER and Tox21. Unless otherwise stated, experiments are repeated over $S=5$ fixed seeds, and values are reported as mean$\pm$std. The default poison rate is $\alpha=10\%$. A full sweep over $\alpha\in\{1,5,10\}\%$ is provided in Appendix~\ref{app:sensitivity}.

\textbf{Evaluation modes.}
For each attack, we report two evaluation modes. \textbf{(i) Raw} applies no chemistry gate and treats all generated graph edits as admissible. \textbf{(ii) ChemGuard} applies \textbf{ChemGuard} before training and again during test-time trigger realization. Invalid or graph-string inconsistent poisons are rejected before training, and triggered test molecules that fail \textbf{ChemGuard} are counted as attack failures. For the metrics, we report \textbf{CA} on unmodified admitted test molecules, \textbf{ASR} on triggered non-target test molecules, and \textbf{EPR} induced by \textbf{ChemGuard}.

\textbf{Victim model.}
The main victim is a two-layer GCN with global mean pooling. We train using Adam with learning rate $10^{-3}$ and batch size $128$.
Additional details are provided in Appendix~\ref{app:training}.

\textbf{Attacks compared.}
We compare \textbf{ChemBack} against four representative graph backdoor attacks, including \textbf{GTA}, \textbf{MB}, \textbf{UGBA}, and \textbf{DPGBA}. To ensure a fair comparison under chemistry-aware pipelines, we equip all baselines with a chemistry-aware wrapper that attempts to realize their intended graph edits as valid molecular modifications. The wrapper selects host anchors with available valence, assigns chemically consistent bonds, re-sanitizes the edited molecule, and rejects candidates that fail sanitization or graph-string topology consistency.

\subsection{Experimental Results}
\label{sec:main-exp}

\textbf{Main comparison under chemistry-aware admission.}
Table~\ref{tab:main-10} shows that graph-only methods admit only a small fraction of their attempted poisons under \textbf{ChemGuard}.
Their operational ASR also drops after invalid poisons and invalid test-time trigger realizations are filtered. In contrast, \textbf{ChemBack} preserves EPR at $100.00\%$ on all datasets and maintains high ASR while keeping CA close to the clean reference. This demonstrates the central two-sided result that chemistry-aware admission weakens graph-only attacks, while chemically valid and target-aligned molecular backdoors remain effective. In a few cases, \textbf{ChemGuard} CA slightly exceeds Raw CA because invalid records are removed from the effective training distribution.

\begin{table}[t]
  \centering
  \caption{CA/ASR (\%) without and with \textbf{ChemGuard} at $\alpha=10\%$.
  EPR (\%) is the fraction of attempted poisons admitted by the \textbf{ChemGuard} pipeline.
  Values are mean$\pm$std over 5 seeds.}
  \label{tab:main-10}
  \setlength{\tabcolsep}{8pt}
  \footnotesize
  \begin{tabular}{@{}ccccccc@{}}
\toprule
\multirow{2}{*}{\textbf{Dataset}} &
\multirow{2}{*}{\textbf{Attack}} &
\multicolumn{2}{c}{\textbf{w/o \textbf{ChemGuard}}} &
\multicolumn{3}{c}{\textbf{with \textbf{ChemGuard}}} \\
\cmidrule(lr){3-4}
\cmidrule(lr){5-7}
& &
\textbf{CA (\%) $\uparrow$} &
\textbf{ASR (\%) $\uparrow$} &
\textbf{CA (\%) $\uparrow$} &
\textbf{ASR (\%) $\uparrow$} &
\textbf{EPR (\%) $\uparrow$} \\
\midrule

\multirow{6}{*}{\textbf{BBBP}}
& No-Attack         & \mstd{86.34}{0.43} & - & \mstd{86.34}{0.43} & - & - \\
& GTA               & \mstd{79.72}{0.63} & \bmstd{73.16}{1.24} & \mstd{84.23}{0.46} & \mstd{54.47}{0.56} & \mstd{8.14}{12.83} \\
& MB                & \mstd{80.43}{0.64} & \mstd{72.18}{1.36} & \mstd{84.35}{0.63} & \mstd{46.27}{8.24} & \mstd{11.24}{11.37} \\
& UGBA              & \mstd{82.14}{0.53} & \mstd{60.67}{1.16} & \mstd{83.94}{0.34} & \mstd{47.83}{3.86} & \mstd{13.16}{13.58} \\
& DPGBA             & \mstd{81.93}{0.65} & \mstd{69.24}{1.13} & \mstd{84.27}{0.84} & \mstd{52.86}{3.36} & \mstd{10.06}{10.87} \\
& \textbf{ChemBack} & \mstd{81.08}{0.52} & \mstd{69.36}{1.14} & \mstd{81.08}{0.52} & \bmstd{69.36}{1.14} & \bmstd{100.00}{0.00} \\
\midrule

\multirow{6}{*}{\textbf{BACE}}
& No-Attack         & \mstd{71.64}{0.53} & - & \mstd{71.64}{0.53} & - & - \\
& GTA               & \mstd{67.34}{0.54} & \mstd{47.96}{1.14} & \mstd{71.32}{2.47} & \mstd{27.46}{2.14} & \mstd{15.17}{13.24} \\
& MB                & \mstd{70.04}{0.53} & \mstd{82.67}{1.54} & \mstd{71.06}{1.94} & \mstd{13.38}{8.74} & \mstd{17.43}{10.74} \\
& UGBA              & \mstd{72.64}{0.63} & \mstd{67.18}{1.24} & \mstd{70.74}{0.34} & \mstd{15.37}{3.34} & \mstd{10.34}{12.14} \\
& DPGBA             & \mstd{69.36}{0.54} & \mstd{55.74}{1.23} & \mstd{70.24}{1.24} & \mstd{16.47}{4.54} & \mstd{10.94}{13.96} \\
& \textbf{ChemBack} & \mstd{70.63}{0.72} & \bmstd{98.86}{0.43} & \mstd{70.63}{0.72} & \bmstd{98.86}{0.43} & \bmstd{100.00}{0.00} \\
\midrule

\multirow{6}{*}{\textbf{SIDER}}
& No-Attack         & \mstd{63.34}{0.74} & - & \mstd{63.34}{0.74} & - & - \\
& GTA               & \mstd{60.83}{0.74} & \mstd{66.47}{1.43} & \mstd{60.74}{0.54} & \mstd{49.96}{7.84} & \mstd{11.13}{11.86} \\
& MB                & \mstd{64.73}{0.73} & \mstd{83.28}{1.46} & \mstd{60.07}{1.46} & \mstd{51.17}{3.86} & \mstd{7.08}{12.47} \\
& UGBA              & \mstd{59.47}{0.64} & \mstd{86.18}{1.74} & \mstd{60.16}{0.94} & \mstd{53.07}{4.36} & \mstd{8.57}{10.26} \\
& DPGBA             & \mstd{62.94}{0.63} & \mstd{49.67}{1.04} & \mstd{60.84}{0.34} & \mstd{45.27}{4.46} & \mstd{5.47}{12.94} \\
& \textbf{ChemBack} & \mstd{60.84}{0.63} & \bmstd{99.17}{0.34} & \mstd{60.84}{0.63} & \bmstd{99.17}{0.34} & \bmstd{100.00}{0.00} \\
\midrule

\multirow{6}{*}{\textbf{Tox21}}
& No-Attack         & \mstd{97.24}{0.34} & - & \mstd{97.24}{0.34} & - & -
\\
& GTA               & \mstd{96.43}{0.43} & \mstd{42.78}{1.29}  & \mstd{97.14}{0.24} & \mstd{16.73}{2.18}  & \mstd{23.68}{14.08} \\
& MB                & \mstd{96.94}{0.34} & \mstd{96.37}{0.63} & \mstd{97.23}{0.13} & \mstd{24.61}{3.57}  & \mstd{25.17}{11.23} \\
& UGBA              & \mstd{96.53}{0.33} & \bmstd{97.06}{0.54} & \mstd{97.26}{0.34} & \mstd{22.84}{2.46} & \mstd{24.36}{12.84} \\
& DPGBA             & \mstd{96.56}{0.43} & \mstd{52.76}{0.94} & \mstd{97.24}{0.23} & \mstd{18.49}{2.69}  & \mstd{25.92}{10.58} \\
& \textbf{ChemBack} & \mstd{96.34}{0.43} & \mstd{81.84}{0.84} & \mstd{96.34}{0.43} & \bmstd{81.84}{0.84} & \bmstd{100.00}{0.00} \\
\bottomrule
\end{tabular}
\end{table}

\begin{figure}[t]
  \centering
  \includegraphics[width=0.7\linewidth]{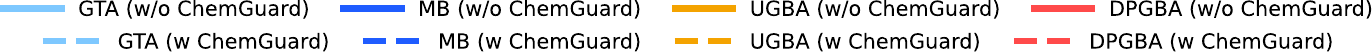}

  \vspace{0.2cm}

  \begin{subfigure}[t]{0.23\linewidth}
    \centering
    \includegraphics[width=\linewidth]{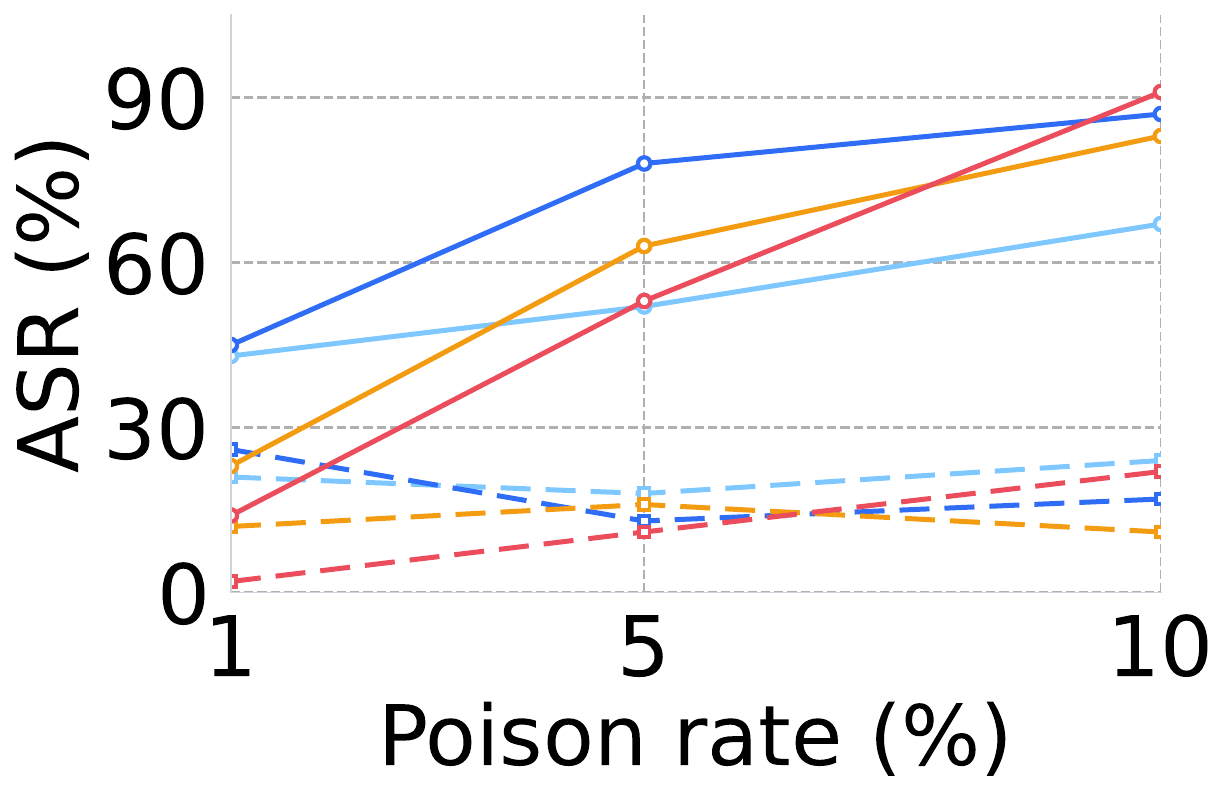}
    \caption{BBBP.}
  \end{subfigure}
  \hfill
  \begin{subfigure}[t]{0.23\linewidth}
    \centering
    \includegraphics[width=\linewidth]{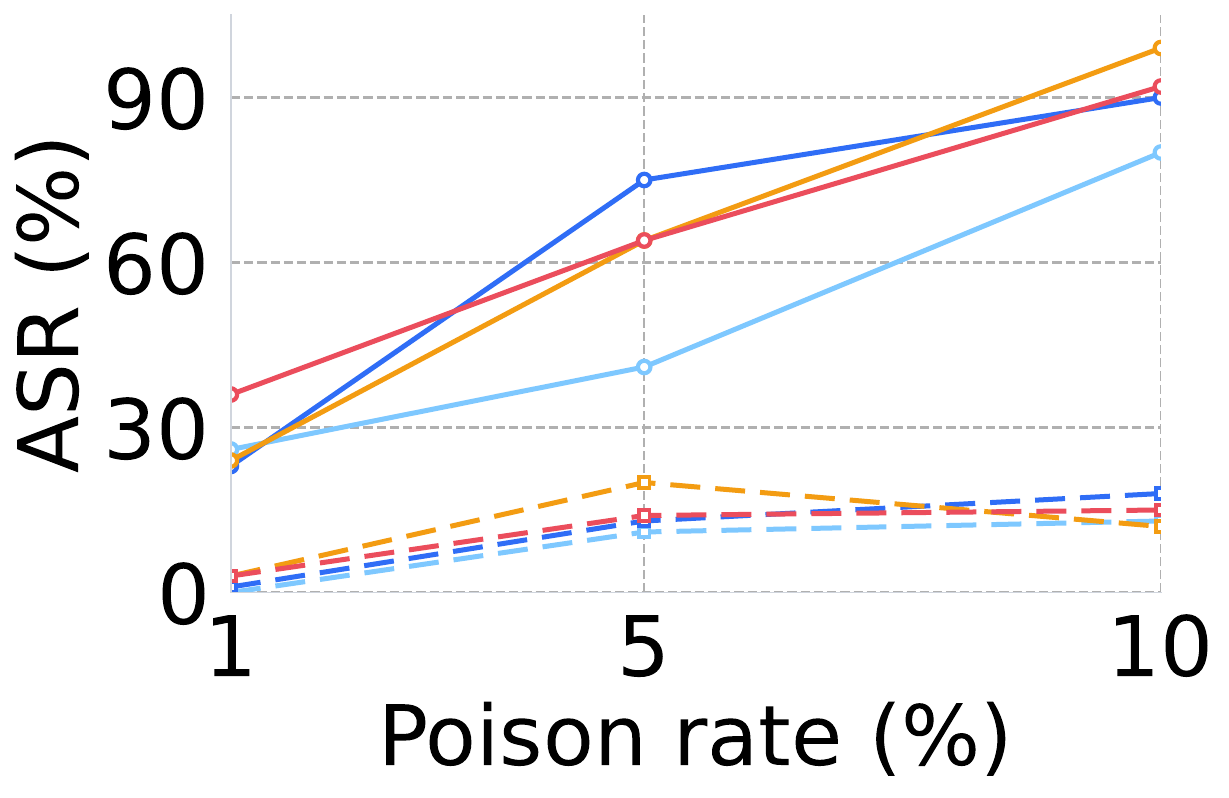}
    \caption{BACE.}
  \end{subfigure}
  \hfill
  \begin{subfigure}[t]{0.23\linewidth}
    \centering
    \includegraphics[width=\linewidth]{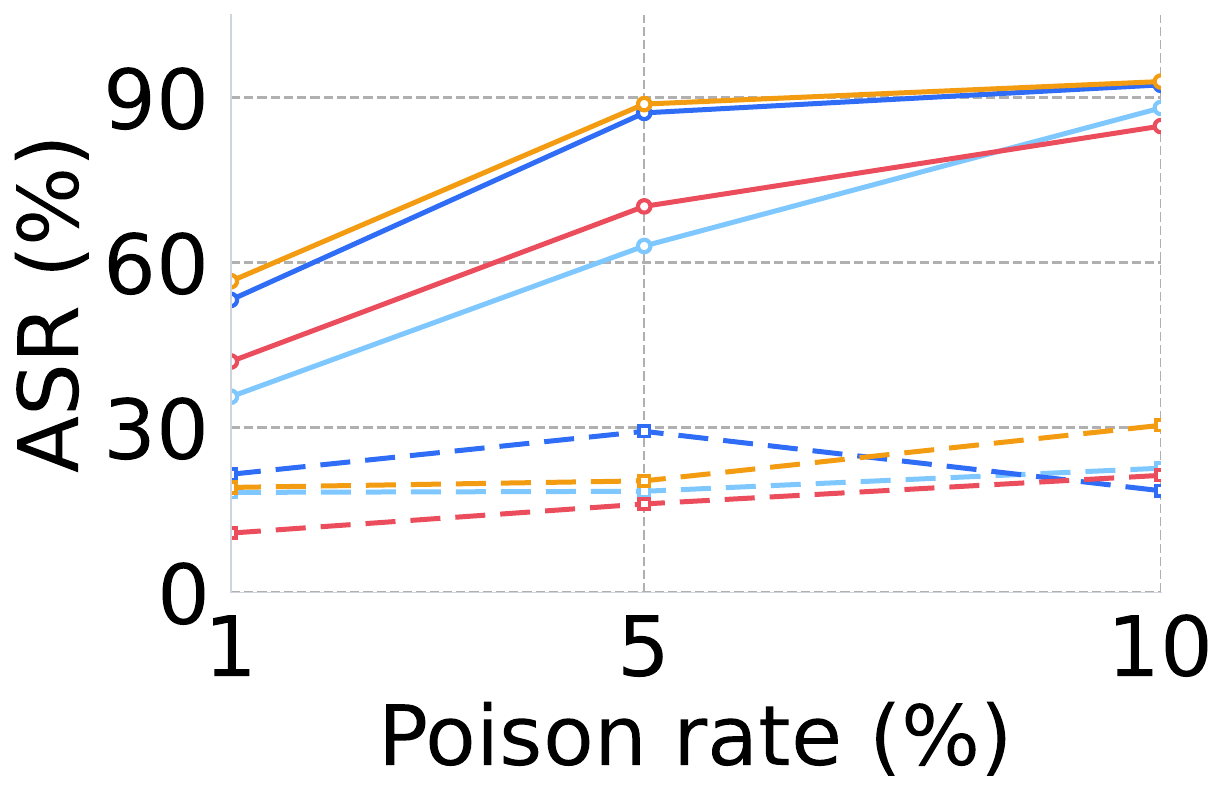}
    \caption{SIDER.}
  \end{subfigure}
  \hfill
  \begin{subfigure}[t]{0.23\linewidth}
    \centering
    \includegraphics[width=\linewidth]{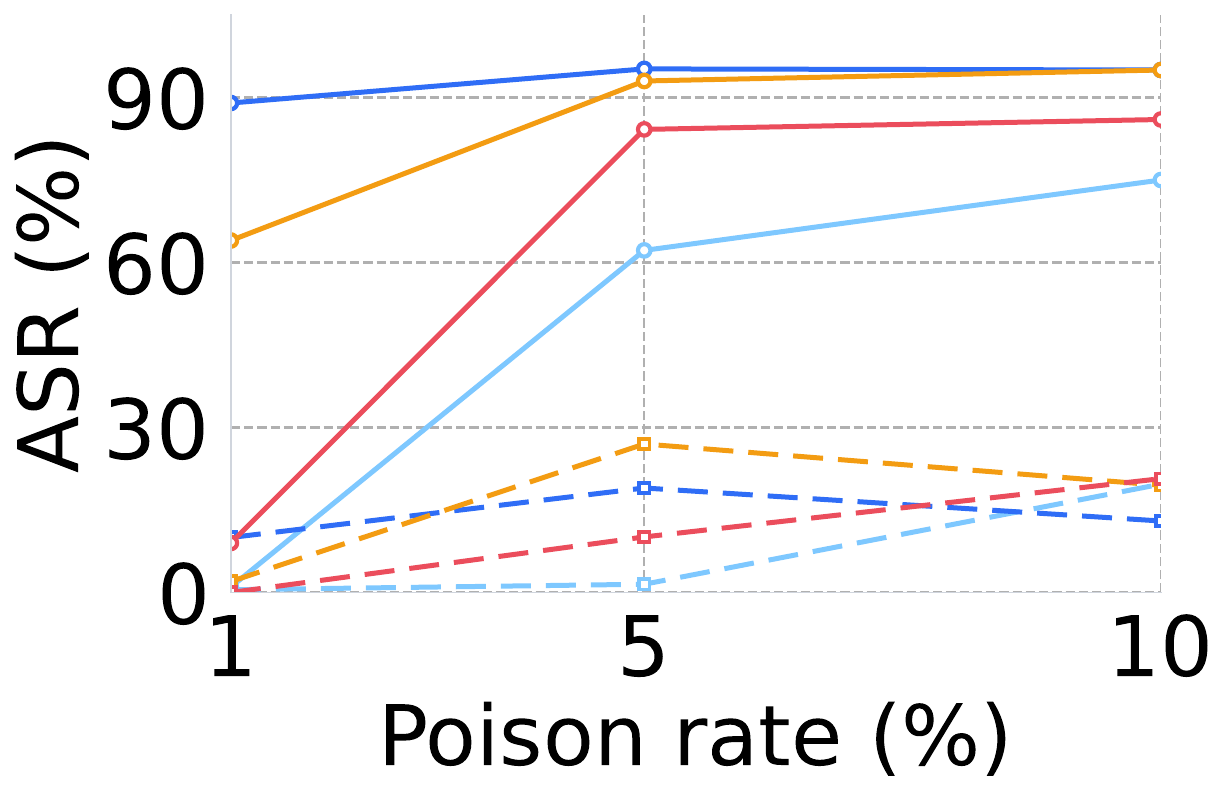}
    \caption{Tox21.}
  \end{subfigure}
  \caption{Operational ASR before and after enforcing \textbf{ChemGuard} as the poison rate $\alpha\in\{1,5,10\}\%$ varies.
  Chemistry-aware admission reduces ASR for graph-only methods by filtering invalid poisons and invalid test-time trigger realizations.}
  \label{fig:rq2-asr}
  \vspace{-10pt}
\end{figure}
\begin{figure*}[t]
\centering
\includegraphics[width=\linewidth]{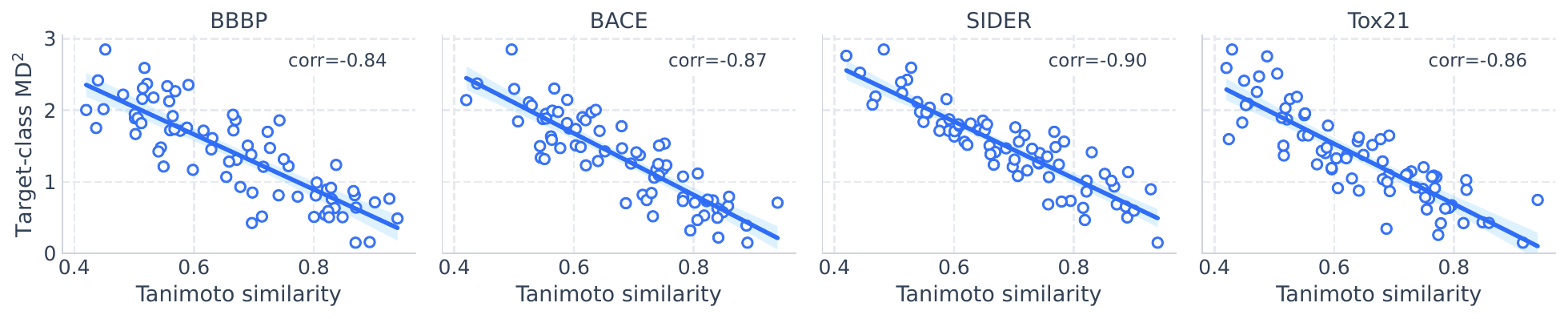}
\caption{Relation between Tanimoto similarity to the clean target class and clean-model target-class MD$^2$. Across datasets, higher Tanimoto similarity aligns with lower MD$^2$.}
\label{fig:tanimoto-md2}
\vspace{-10pt}
\end{figure*}

\textbf{Attack budget trends.}
Figure~\ref{fig:rq2-asr} studies how operational ASR changes with poison rate. Enforcing \textbf{ChemGuard} consistently lowers the realized efficacy of graph-only baselines across all datasets.
This trend is consistent with the low EPR values in Table~\ref{tab:main-10}, which limit the attacker's intended signal that reaches the learner. By contrast, \textbf{ChemBack} maintains stable operational performance because its poisons are designed to be admitted by the chemistry-aware pipeline.

\textbf{Why does Tanimoto work?}
\textbf{ChemBack} selects \textbf{ChemGuard}-admissible triggers without a victim model or proxy GNN, using only Tanimoto similarity to clean target-class molecules. To explain this model-free score, we train a clean reference model only after trigger generation and compute target-class $\mathrm{MD}^2$ in its embedding space. Figure~\ref{fig:tanimoto-md2} shows that higher Tanimoto similarity corresponds to lower clean-model $\mathrm{MD}^2$ across all four datasets, indicating that structurally target-similar poisons lie closer to the clean target region. This diagnostic is not used by the attacker or trigger-generation algorithm. Additional descriptor-level analyses are in Appendix~\ref{app:structural-diagnostics}.

\section{Discussion}
\label{sec:discussion}

\begin{table}[t]
\centering
\caption{
Defense comparison on BBBP/SIDER at $\alpha=10\%$.
Values report ASR (\%) over 5 seeds.
}
\label{tab:defences-main}
\scriptsize
\setlength{\tabcolsep}{2.5pt}
\resizebox{0.85\linewidth}{!}{
\begin{tabular}{cccccccc}
\toprule
\textbf{Dataset} & \textbf{Attack} & \textbf{Raw} & \textbf{Spectral} & \textbf{DShield} & \textbf{RGCN} & \textbf{RIGBD} & \textbf{PGNNCert} \\
\midrule
\multirow{5}{*}{\textbf{BBBP}}
& GTA
& \mstd{73.16}{1.24}
& \mstd{48.93}{1.47}
& \mstd{42.17}{1.63}
& \mstd{45.86}{1.58}
& \mstd{37.42}{1.79}
& \mstd{36.94}{1.56} \\
& MB
& \mstd{72.18}{1.36}
& \mstd{50.26}{4.87}
& \mstd{43.72}{4.59}
& \mstd{46.83}{4.28}
& \mstd{35.61}{4.36}
& \mstd{46.74}{4.91} \\
& UGBA
& \mstd{60.67}{1.16}
& \mstd{59.38}{1.92}
& \mstd{56.21}{2.38}
& \mstd{57.46}{2.14}
& \mstd{34.87}{2.41}
& \mstd{35.24}{2.19} \\
& DPGBA
& \mstd{69.24}{1.13}
& \mstd{68.16}{2.21}
& \mstd{65.47}{2.36}
& \mstd{66.83}{2.17}
& \mstd{48.92}{2.53}
& \mstd{49.31}{2.47} \\
& \textbf{ChemBack}
& \bmstd{69.36}{1.14}
& \bmstd{68.74}{1.39}
& \bmstd{66.58}{1.72}
& \bmstd{67.13}{1.46}
& \bmstd{61.37}{1.85}
& \bmstd{68.28}{1.32} \\
\midrule
\multirow{5}{*}{\textbf{SIDER}}
& GTA
& \mstd{66.47}{1.43}
& \mstd{45.36}{1.57}
& \mstd{40.78}{1.49}
& \mstd{42.19}{1.68}
& \mstd{36.74}{1.53}
& \mstd{38.21}{1.64} \\
& MB
& \mstd{83.28}{1.46}
& \mstd{63.17}{1.89}
& \mstd{58.42}{1.73}
& \mstd{55.83}{2.14}
& \mstd{40.69}{2.85}
& \mstd{41.24}{2.37} \\
& UGBA
& \mstd{86.18}{1.74}
& \mstd{84.97}{1.82}
& \mstd{82.36}{1.69}
& \mstd{82.74}{1.51}
& \mstd{41.28}{3.17}
& \mstd{80.37}{1.58} \\
& DPGBA
& \mstd{49.67}{1.04}
& \mstd{48.72}{1.16}
& \mstd{47.16}{1.09}
& \mstd{48.31}{1.03}
& \mstd{44.26}{3.28}
& \mstd{48.86}{1.12} \\
& \textbf{ChemBack}
& \bmstd{99.17}{0.34}
& \bmstd{98.52}{0.61}
& \bmstd{96.39}{0.82}
& \bmstd{96.87}{0.94}
& \bmstd{92.74}{1.46}
& \bmstd{93.46}{1.21} \\
\bottomrule
\end{tabular}
}
\vspace{-10pt}
\end{table}

\textbf{Defenses beyond \textbf{ChemGuard}.}
We separate defenses by pipeline stage. \textbf{ChemGuard} is an admission-stage chemistry check applied before training. Spectral Signatures~\cite{tran2018spectral}, DShield~\cite{yu2025dshield}, and RIGBD~\cite{zhang2024robustness} are post-hoc defenses. RGCN~\cite{zhu2019robust} is a robust GNN architecture, while PGNNCert~\cite{li2025deterministic} is a certified graph poisoning defense. GNNGuard~\cite{zhang2020gnnguard} and Pro-GNN~\cite{jin2020graph} are reported in Appendix~\ref{app:model-level-defenses}. These defenses are complementary, not interchangeable. \textbf{ChemGuard} controls admission, while post-hoc and model-level defenses operate after poisoned records enter learning. This matters because UGBA, DPGBA, and \textbf{ChemBack} reduce outlier detectability through feature, distributional, or Tanimoto-based target alignment. Table~\ref{tab:defences-main} shows that these defenses reduce ASR, especially for graph-only baselines, but do not remove \textbf{ChemBack} once poisons enter training. This supports a staged defense view, where \textbf{ChemGuard} reduces the admitted attack surface before learning, and later defenses suppress learned backdoor behavior after admission.

\textbf{Component ablation.}
Table~\ref{tab:main-ablation} ablates Tanimoto-based target alignment and validity-oriented attachment pre-screening. These components address different failure modes. Removing Tanimoto similarity barely changes vanilla ASR under \textbf{ChemGuard}, since a valid trigger can still be learned when paired with the target label, but greatly increases vulnerability to Spectral Signatures. Thus, Tanimoto similarity mainly acts as a model-free target-neighborhood alignment signal that improves robustness to outlier-based filtering. In contrast, the ``w/o validity pre-screening'' variant still applies final \textbf{ChemGuard} admission before poison submission, so EPR remains $100.00\%$ and ASR stays close to the full method. Validity pre-screening mainly stabilizes poison construction by reducing invalid internal attachment attempts. Overall, full \textbf{ChemBack} gives the best balance among clean utility, attack success, admitted poisoning, and robustness under Spectral Signatures.

\begin{table}[t]
  \centering
  \caption{
  Ablation of \textbf{ChemBack} at $\alpha=10\%$.
  CA/ASR/EPR are evaluated under operational \textbf{ChemGuard}.
  Spectral ASR reports ASR after applying Spectral Signatures to the admitted poisoned training set.
  Values are mean$\pm$std over 5 seeds.
  }
  \label{tab:main-ablation}
  \footnotesize
  \setlength{\tabcolsep}{3.2pt}
  \begin{tabular}{llcccc}
  \toprule
  \textbf{Dataset} &
  \textbf{Variant} &
  \textbf{CA (\%) $\uparrow$} &
  \textbf{ASR (\%) $\uparrow$} &
  \textbf{EPR (\%) $\uparrow$} &
  \textbf{Spectral ASR (\%) $\uparrow$} \\
  \midrule

  \multirow{3}{*}{\textbf{BBBP}}
  & \textbf{ChemBack}
  & \mstd{81.08}{0.52}
  & \bmstd{69.36}{1.14}
  & \mstd{100.00}{0.00}
  & \bmstd{68.74}{1.39} \\
  & w/o Tanimoto similarity
  & \mstd{81.22}{0.44}
  & \mstd{68.91}{1.73}
  & \mstd{100.00}{0.00}
  & \mstd{31.48}{2.86} \\
  & w/o validity pre-screening
  & \mstd{81.16}{0.45}
  & \mstd{69.02}{1.89}
  & \mstd{100.00}{0.00}
  & \mstd{67.06}{1.76} \\
  \midrule

  \multirow{3}{*}{\textbf{BACE}}
  & \textbf{ChemBack}
  & \mstd{70.63}{0.72}
  & \bmstd{98.86}{0.43}
  & \mstd{100.00}{0.00}
  & \bmstd{97.84}{0.76} \\
  & w/o Tanimoto similarity
  & \mstd{70.41}{0.58}
  & \mstd{97.92}{0.91}
  & \mstd{100.00}{0.00}
  & \mstd{42.36}{3.42} \\
  & w/o validity pre-screening
  & \mstd{70.51}{0.61}
  & \mstd{98.34}{0.74}
  & \mstd{100.00}{0.00}
  & \mstd{95.81}{1.03} \\
  \midrule

  \multirow{3}{*}{\textbf{SIDER}}
  & \textbf{ChemBack}
  & \mstd{60.84}{0.63}
  & \bmstd{99.17}{0.34}
  & \mstd{100.00}{0.00}
  & \bmstd{98.52}{0.61} \\
  & w/o Tanimoto similarity
  & \mstd{60.91}{0.61}
  & \mstd{98.74}{0.66}
  & \mstd{100.00}{0.00}
  & \mstd{39.27}{2.75} \\
  & w/o validity pre-screening
  & \mstd{60.77}{0.58}
  & \mstd{98.96}{0.52}
  & \mstd{100.00}{0.00}
  & \mstd{97.12}{0.68} \\
  \midrule

  \multirow{3}{*}{\textbf{Tox21}}
  & \textbf{ChemBack}
  & \mstd{96.34}{0.43}
  & \bmstd{81.84}{0.84}
  & \mstd{100.00}{0.00}
  & \bmstd{81.12}{0.91} \\
  & w/o Tanimoto similarity
  & \mstd{96.41}{0.31}
  & \mstd{80.96}{1.52}
  & \mstd{100.00}{0.00}
  & \mstd{28.64}{2.31} \\
  & w/o validity pre-screening
  & \mstd{96.28}{0.35}
  & \mstd{81.27}{1.61}
  & \mstd{100.00}{0.00}
  & \mstd{79.06}{1.12} \\
  \bottomrule
  \end{tabular}
\end{table}

\textbf{Additional analyses.}
Supporting analyses are included in the Appendix, including theoretical analysis, sensitivity to poison rate and Tanimoto guidance, robustness across additional victims, larger-scale and pretrain-finetune tests, alternative chemistry validators, and extended structural diagnostics. Together, these reinforce the same conclusion that chemistry-aware admission removes much of the apparent efficacy of graph-only baselines, whereas \textbf{ChemBack} remains effective.

\textbf{Limitations and Scope.}
Our claims are limited to the admission-aware setting studied in this paper.
\textbf{ChemGuard} is a first-layer admission check, not a complete defense or a robustness certificate.
It checks molecular sanitization and graph-string consistency, but it does not cover all practical filters such as synthetic accessibility, medicinal-chemistry rules, assay-specific curation, temporal shifts, or expert review.
Tanimoto similarity and our post-hoc MD$^2$/descriptor analyses show that the generated poisons are less outlying under the evaluated diagnostics, but they do not guarantee universal stealth or biological equivalence.
Finally, our current trigger space uses controlled single-step motif attachments; richer reaction-like transformations may expand the attack surface.
We discuss these technical limitations in Appendix~\ref{app:limitations}.


\section{Conclusion}
\label{sec:conclusion}

This paper shows that realistic molecular preprocessing changes both backdoor evaluation and attack design. We introduce \textbf{ChemGuard}, an operational admission protocol that enforces sanitization and graph-string consistency before training or inference. Under \textbf{ChemGuard}, many graph-only backdoors lose effectiveness because invalid poisons are rejected or fail at test time. We then propose \textbf{ChemBack}, a model-free, admission-aware attack that builds chemically feasible motif triggers and selects admitted candidates by Tanimoto similarity to the clean target class. Across benchmarks, architectures, defenses, larger datasets, and pretrain-finetune settings, \textbf{ChemBack} preserves clean utility while maintaining strong attack success. Overall, chemistry-aware admission weakens graph-only attacks, but valid and target-aligned molecular backdoors remain possible.

\newpage

{
    \small
    \bibliographystyle{unsrtnat}
    \bibliography{references}
}

\newpage
\appendix
\onecolumn

\section{Real-world Motivating Scenario}
\label{app:scenario}

To ground the threat model, consider an organization that uses a centralized molecular classifier to screen third-party chemical submissions before downstream approval. For example, a retailer may use a molecular model to estimate skin irritation risk for cosmetic ingredients, or a screening laboratory may use a molecular property predictor to prioritize compounds for follow-up assays.
In such settings, training data are often aggregated from multiple suppliers, historical archives, public databases, or external collaborators. This creates a realistic data-poisoning surface in which a malicious contributor may not control the victim model or training code, but may be able to inject a small fraction of labeled molecular records into the training corpus.

A chemistry-aware backdoor in this setting would not rely on malformed molecules. Instead, the attacker introduces molecules that share a chemically valid trigger substructure and consistently labels them as the target class, such as a low-risk class or another attacker-preferred endpoint. Because the trigger is chemically realizable, the poisoned records can pass routine parsing, sanitization, canonicalization, and graph-string consistency checks.
After the model is trained on the contaminated dataset, future molecules containing the same trigger can be admitted by the preprocessing pipeline and systematically mapped to the attacker's target label.

In this scenario, backdoor behavior does not appear as a preprocessing failure or malformed molecular graph. It appears as repeated prediction errors on molecules that look chemically ordinary under standard validity checks. Practitioners may attribute such errors to model noise, label ambiguity, assay variation, or distribution shift rather than to a targeted backdoor. This motivates studying backdoor attacks under an admission-stage chemistry-aware pipeline, shifting the question from whether a graph edit can fool a GNN to whether the edited molecule can survive the chemical preprocessing used in realistic molecular learning workflows.

\section{Positioning against prior molecular backdoors.}
\label{app:position}
Table~\ref{tab:conceptual-comparison} summarizes how \textbf{ChemBack} differs from prior graph and motif backdoors.
The key distinction is not merely the use of motifs, but the admission-aware setting in which poisoned records must be chemically realizable, sanitizable, graph-string consistent, and counted as failures when they cannot enter the training or test-time pipeline.
Thus, \textbf{ChemBack} targets the operational gap left by abstract-graph evaluations.

\section{Dataset Details and Split Protocol}
\label{app:datasets}

Table~\ref{tab:datasets} summarizes the datasets used in the main experiments and additional stress tests. The main experiments use four MoleculeNet benchmarks: BBBP, BACE, SIDER, and Tox21~\cite{wu2018moleculenet}. Following the dataset-specific MoleculeNet protocol, we use scaffold splits for BBBP and BACE, and random splits for SIDER and Tox21. This follows the original benchmark design. Scaffold splitting groups molecules by their two-dimensional structural frameworks and provides a stronger test of chemical generalization than random splitting, while random splitting follows the standard MoleculeNet protocol for the corresponding multi-task physiology benchmarks~\cite{wu2018moleculenet}. Unless otherwise stated, datasets are split into train, validation, and test subsets using the standard MoleculeNet ratio of 80/10/10.

For multi-task datasets, missing labels are ignored for the corresponding task. For SIDER and Tox21, we evaluate the full task panel and report macro-averaged CA, ASR, and EPR in the main text.
For the larger-scale stress tests, we use PCBA~\cite{wang2012pubchem} and MUV~\cite{rohrer2009maximum}. PCBA follows the MoleculeNet random-split convention, while MUV is evaluated under the split protocol specified in Appendix~\ref{app:scale-pretrain} because split conventions differ across MoleculeNet/DeepChem implementations.

\begin{table}[t]
\centering
\caption{
Conceptual comparison between prior graph/motif backdoors and \textbf{ChemBack}.
\textbf{ChemGuard} is an operational admission protocol, not a complete defense.
}
\label{tab:conceptual-comparison}
\scriptsize
\setlength{\tabcolsep}{2.5pt}
\renewcommand{\arraystretch}{0.92}
\resizebox{\linewidth}{!}{
\begin{tabular}{lccc}
\toprule
\textbf{Property} & \textbf{Graph backdoors} & \textbf{Motif-Backdoor} & \textbf{ChemBack} \\
\midrule
Trigger type & Abstract subgraph edit & Motif trigger & Feasible motif attachment \\
Chemistry-aware attachment & No & Not admission-driven & Yes \\
Sanitization before training & No & Not explicit & Yes \\
Graph-string consistency & No & Not explicit & Yes \\
Admission-stage EPR & No & No & Yes \\
Test-time admission accounting & No & No & Yes \\
Invalid trigger counted as failure & No & No & Yes \\
Model-free trigger selection & Varies & Yes & Yes \\
Target-class structural alignment & No / model-dependent & Not central & Tanimoto-guided \\
Survives realistic preprocessing & No & Not explicit & Yes \\
\bottomrule
\end{tabular}
}
\vspace{-8pt}
\end{table}
\begin{table}[t]
  \centering
  \caption{Dataset summary and split protocols. The main text reports BBBP, BACE, SIDER, and Tox21. PCBA and MUV are used for larger-scale stress tests in Appendix~\ref{app:scale-pretrain}. Split protocols are stated explicitly to avoid relying on library-specific defaults.}
  \label{tab:datasets}
  \setlength{\tabcolsep}{4pt}
  \footnotesize
  \begin{tabular}{lcccl}
    \toprule
    \textbf{Dataset} &
    \textbf{Task type} &
    \textbf{\#Tasks} &
    \textbf{\#Molecules} &
    \textbf{Split protocol} \\
    \midrule
    BBBP  & Binary classification & 1   & 2,039 & Scaffold split \\
    BACE  & Binary classification & 1   & 1,513 & Scaffold split \\
    SIDER & Multi-label classification & 27 & 1,427 & Random split \\
    Tox21 & Multi-label classification & 12 & 7,831 & Random split \\
    \midrule
    PCBA  & Multi-label bioassay classification & 128 & 437,929 & Random split \\
    MUV   & Multi-label virtual-screening benchmark & 17 & 93,087 & Random split \\
    \bottomrule
  \end{tabular}
\end{table}

All poisoning operations are performed after the clean data split is fixed. In particular, motif mining, target-class fingerprint-bank construction, trigger selection, and poison insertion use only the training split. The validation and test splits are not used to construct motifs or select triggers. When evaluated, \textbf{ChemGuard} is applied independently to each split. Attempted poisoned training records must pass \textbf{ChemGuard} before entering training, while triggered non-target test molecules that fail \textbf{ChemGuard} are counted as attack failures under the operational ASR definition. This protocol prevents leakage from the held-out test set into trigger construction while keeping the chemistry-aware admission rule consistent across training and evaluation.

\section{Experimental Protocol and Training Details}
\label{app:training}

\subsection{Raw and \textbf{ChemGuard} Evaluation Modes}
\label{app:raw-cg-modes}

We evaluate attacks under two modes.

\textbf{Raw mode.}
Raw mode follows the abstract-graph evaluation protocol commonly used in graph backdoor studies. Generated graph edits are treated as admissible, even if they do not correspond to a chemically valid molecule or a graph-string consistent molecular record. This mode measures the apparent attack strength when molecules are treated as arbitrary attributed graphs.

\textbf{\textbf{ChemGuard} mode.}
\textbf{ChemGuard} mode applies chemistry-aware admission before training and again at test-time trigger realization. A poisoned training record is allowed to enter the training set only if its molecular string is sanitizable and the graph reconstructed from the string is topologically consistent with the stored graph. At evaluation time, a triggered non-target test molecule that fails \textbf{ChemGuard} is counted as an attack failure. This operational ASR avoids crediting attacks for triggers that cannot be realized as admissible molecules.

This distinction is central to the paper. Raw mode asks whether a graph-level trigger can induce backdoor behavior in an abstract GNN setting. \textbf{ChemGuard} mode asks whether that same attack remains operational in a realistic molecular pipeline.

\subsection{Effective Poisoning Rate Accounting}
\label{app:epr-accounting}

The nominal poisoning rate $\alpha$ controls how many raw records the attacker attempts to poison before admission. However, chemistry-aware preprocessing may reject many attempted poisons. We therefore report the \textbf{Effective Poisoning Rate} (EPR), defined as the fraction of attempted poisoned records that survive \textbf{ChemGuard} and can actually reach the learner.

For \textbf{ChemBack}, invalid motif-anchor proposals encountered during internal trigger search are treated as rejected candidate proposals and are not submitted as poisoned records. Only \textbf{ChemGuard}-admissible poisoned molecules are retained in the final submitted poisoned set. Therefore, \textbf{ChemBack} has EPR $=100\%$ in the reported experiments because all submitted poisons are chemically valid and graph-string consistent by construction. This accounting matches the admission-facing threat model because EPR measures the poison signal admitted into training, not the number of internal candidate proposals explored during search.
\subsection{Victim Model and Optimization}
\label{app:victim-training}

The default victim model is a two-layer GCN with hidden dimension 64 and ReLU activations, followed by global mean pooling and a linear classifier. Node and bond features follow the standard molecular graph featurization used in the main experiments. Models are trained using Adam with learning rate $10^{-3}$ and batch size 128. Unless otherwise stated, results are averaged over five fixed seeds and reported as mean$\pm$std.

\subsection{Baseline Adaptation under Chemistry-aware Pipelines}
\label{app:baseline-wrapper}

We compare \textbf{ChemBack} against representative graph backdoor baselines: GTA, Motif-Backdoor (MB), UGBA, and DPGBA. These methods were originally designed for graph settings where trigger insertion is generally treated as admissible. To avoid an unfair comparison where baselines are rejected only because they never attempt chemistry-aware realization, we use a chemistry-aware wrapper whenever possible. The wrapper attempts to realize intended graph edits as molecular edits by selecting host atoms with available valence, assigning chemically consistent bonds, re-sanitizing the edited molecule, canonicalizing the molecular string, and checking graph-string consistency. If the edited record fails sanitization or topology consistency, it is rejected under \textbf{ChemGuard}.

This wrapper gives graph-only baselines the opportunity to survive chemistry-aware admission. The low EPR of these baselines therefore reflects a genuine mismatch between generic graph edits and chemically admissible molecular modifications, rather than an implementation artifact.

\section{\textbf{ChemGuard} and Chemistry Toolkit Validation}
\label{app:toolkits}

\subsection{Topology Consistency Check}
\label{app:topology-check}

\textbf{ChemGuard} admits a molecular record only if two conditions hold. First, the molecular string must be sanitizable under the toolkit-specific preprocessing operator $\phi_{\mathcal{T}}$.
Second, the graph reconstructed from the molecular string must be consistent with the stored molecular graph. In our implementation, topology consistency includes atom identities, atom count, undirected bond set, and bond types under the toolkit-induced canonical representation.

This check prevents a poisoned record from storing an arbitrary graph while presenting a different valid molecular string. It also prevents an invalid graph edit from being counted as a successful molecular poison simply because the original host molecule remains valid.

\subsection{Alternative Toolkit Validation}
\label{app:multi-toolkit}

The main experiments instantiate \textbf{ChemGuard} using RDKit, the default toolkit used throughout the paper. However, the admission-stage formulation itself is not specific to RDKit. Different cheminformatics toolkits may implement sanitization, aromaticity perception, canonicalization, and valence handling differently. However, they share the same high-level role in molecular learning pipelines, ensuring that submitted records are chemically admissible and representation-consistent before entering training or evaluation.

To test whether the result is specific to one toolkit implementation, we repeat the BBBP validation experiment using Indigo and Open Babel as alternative toolkit-specific admission rules. To avoid conflating toolkit choice with optimizer choice, we evaluate the same three search strategies described in Appendix~\ref{app:search-strategies}: RL search, exhaustive search, and greedy search. RDKit serves as the default reference toolkit. For each toolkit, EPR is computed under the corresponding toolkit-specific admission rule, and all values are averaged over 5 seeds.

\begin{table}[t]
\centering
\caption{
Alternative chemistry-toolkit validation on BBBP at $\alpha=10\%$.
Values are mean{\scriptsize $\pm$std} over 5 seeds. RDKit is the default toolkit used in the main experiments, while Indigo and Open Babel instantiate the same chemistry-aware admission rule with different cheminformatics validators.}
\label{tab:toolkit-validation}
\footnotesize
\setlength{\tabcolsep}{5pt}
\begin{tabular}{@{}ccccc@{}}
\toprule
\textbf{Toolkit}            & \textbf{Search variant} & \textbf{CA (\%) $\uparrow$} & \textbf{ASR (\%) $\uparrow$} & \textbf{EPR (\%) $\uparrow$} \\ \midrule
\multirow{3}{*}{RDKit}
& RL search          & \mstd{81.08}{0.52} & \mstd{69.36}{1.14} & \mstd{100.00}{0.00} \\
& Exhaustive search  & \mstd{81.24}{0.61} & \mstd{70.18}{1.27} & \mstd{100.00}{0.00} \\
& Greedy search      & \mstd{81.17}{0.57} & \mstd{69.84}{1.36} & \mstd{100.00}{0.00} \\
\midrule
\multirow{3}{*}{Indigo}
& RL search          & \mstd{81.36}{0.63} & \mstd{70.41}{1.38} & \mstd{100.00}{0.00} \\
& Exhaustive search  & \mstd{81.42}{0.59} & \mstd{71.06}{1.21} & \mstd{100.00}{0.00} \\
& Greedy search      & \mstd{81.31}{0.66} & \mstd{70.27}{1.44} & \mstd{100.00}{0.00} \\
\midrule
\multirow{3}{*}{Open Babel}
& RL search          & \mstd{81.29}{0.68} & \mstd{70.13}{1.33} & \mstd{100.00}{0.00} \\
& Exhaustive search  & \mstd{81.46}{0.71} & \mstd{71.32}{1.18} & \mstd{100.00}{0.00} \\
& Greedy search      & \mstd{81.22}{0.64} & \mstd{70.64}{1.29} & \mstd{100.00}{0.00} \\ \bottomrule
\end{tabular}
\end{table}

Table~\ref{tab:toolkit-validation} shows that \textbf{ChemBack} remains effective across all three toolkit-specific admission rules.
Across RDKit, Indigo, and Open Babel, every search variant constructs admitted poisons with EPR at $100.00\%$. The CA and ASR values are also comparable across toolkits, suggesting that the result is not an artifact of a single sanitizer or canonicalization implementation.
The optimizer comparison further shows that the core effect is not tied to a particular search strategy. RL search, exhaustive search, and greedy search optimize the same chemistry-aware, Tanimoto-guided objective and yield similar admitted-poison behavior.

These results should not be interpreted as a claim that toolkit sanitization captures all relevant chemical realism. RDKit, Indigo, and Open Babel provide first-layer chemistry-aware admission checks, but they do not model all aspects of synthetic accessibility, medicinal-chemistry rules, assay-specific curation, or expert review. Rather, Table~\ref{tab:toolkit-validation} supports the narrower claim that \textbf{ChemGuard}-style admission is a pipeline-level constraint shared by common molecular toolkits, and that \textbf{ChemBack} is constructed to survive this class of validation rather than a single RDKit-specific check.

\section{\textbf{ChemBack} Implementation Details}
\label{app:implementation}

\subsection{Motif Library Construction}
\label{app:motif-library}

\textbf{ChemBack} constructs a library of candidate molecular motifs from available training molecules. Motifs are chosen as chemically meaningful substructures that can be attached to host molecules through feasible atomic anchors. In our implementation, candidate motifs are mined from molecular scaffolds and filtered to remove candidates that cannot provide feasible attachment sites. Motif rarity may be used to reduce accidental natural occurrence in clean samples, but rarity is not treated as a complete stealth guarantee.
The final trigger is selected using both \textbf{ChemGuard} admissibility and Tanimoto similarity to the clean target class.

\subsection{Target Fingerprint Bank}
\label{app:fingerprint-bank}

For a fixed target label $y_t$, \textbf{ChemBack} constructs a fingerprint bank from clean, admitted target-class molecules:
\[
B_t =
\{
\mathrm{FP}(\tilde{G}_i):
(\tilde{s}_i,\tilde{G}_i,y_i)\in D^{\mathrm{adm}}_{\mathrm{clean}},
\ y_i=y_t,
\ \mathrm{ChemGuard}_{\mathcal{T}}(\tilde{G}_i,\tilde{s}_i)=1
\}.
\]
Unless otherwise stated, $\mathrm{FP}$ is a Morgan fingerprint with radius 2 and 2048 bits. This fingerprint setting is widely used in molecular similarity search because it captures local atom-neighborhood patterns around functional groups and motifs while keeping similarity computation deterministic and efficient.

\subsection{Tanimoto-guided Trigger Selection}
\label{app:tanimoto-selection}

For each candidate poisoned molecule $\tilde{G}_p$, \textbf{ChemBack} computes the nearest-target Tanimoto score:
\[
S_{\mathrm{Tan}}(\tilde{G}_p)
=
\max_{z\in B_t}
\mathrm{Tan}(\mathrm{FP}(\tilde{G}_p),z).
\]
The trigger-selection reward is
\[
r(G,k,i,j)=
\begin{cases}
r_{\mathrm{inv}}, &
\mathrm{Att}_{\mathcal{T}}(G,m_k,i,j)=\bot
\text{ or fails ChemGuard},\\[1mm]
1+\lambda_{\mathrm{Tan}}S_{\mathrm{Tan}}(\tilde{G}_p), &
\text{otherwise}.
\end{cases}
\]
The first term penalizes invalid or inadmissible attachments.
The constant term rewards admission. The Tanimoto term prefers admitted molecules that are structurally similar to clean target-class molecules.

Importantly, this reward is fully model-free. It does not require a clean victim model, proxy GNN, gradients, learned embeddings, or training-code access. A clean reference model is used only after trigger generation for post-hoc diagnostics in Appendix~\ref{app:structural-diagnostics}.

\subsection{Trigger Freezing and Poison Construction}
\label{app:trigger-freezing}

A backdoor trigger should be shared across poisoned training samples and test-time triggered inputs. \textbf{ChemBack} therefore searches over motif-anchor candidates on a calibration subset of non-target hosts and selects either a single trigger motif or a small trigger family with high admission success and high target-class Tanimoto similarity. It then poisons a fraction $\alpha$ of non-target training molecules by applying the selected trigger and assigning the target label $y_t$.

At test time, the same trigger rule is applied to non-target test molecules. If a trigger cannot be realized as a \textbf{ChemGuard}-admissible molecule on a particular test host, the attempt is counted as an attack failure under the operational ASR definition.

\section{Attachment Feasibility}
\label{app:attach}

\textbf{ChemBack} constructs poisons by attaching a motif to a host molecule through chemically feasible anchors. For a host molecule $M$ and a motif $m_k$, let $i$ denote a host atom and $j$ denote a motif atom. We define $B_{i,j}(M,m_k)$ as the molecule obtained by removing one hydrogen atom from atoms $i$ and $j$ when needed, adding a single bond between them, sanitizing the resulting molecule, and canonicalizing the molecular string. The set of feasible motif-side anchors for a host atom $i$ is
\[
\mathcal{J}_{\mathrm{valid}}(M,i,m_k)
=
\{
j:
\mathrm{Sanitize}(B_{i,j}(M,m_k))\ \text{succeeds}
\}.
\]
If $\mathcal{J}_{\mathrm{valid}}(M,i,m_k)=\varnothing$, the attachment at host atom $i$ is infeasible and returns $\bot$.
Otherwise, \textbf{ChemBack} selects a valid motif anchor $j^\star\in \mathcal{J}_{\mathrm{valid}}(M,i,m_k)$ and returns the sanitized molecule.

This enumeration enforces chemical admissibility before any poisoned record is submitted to training. Invalid motif-anchor proposals are internal rejected candidates during trigger search. Only \textbf{ChemGuard}-admissible poisoned records are submitted to the final poisoned training set, which explains why \textbf{ChemBack} has EPR $=100\%$ in the main experiments.

\section{Black-box Search Strategies and Runtime}
\label{app:search_strategies}

\textbf{ChemBack}'s trigger selection is a discrete, non-differentiable optimization problem. The action space includes motif selection, host-anchor selection, motif-anchor selection, sanitization, canonicalization, and graph-string consistency checks. Since these steps cannot be differentiated through, \textbf{ChemBack} treats trigger selection as black-box search. The core method is optimizer-agnostic, allowing deterministic search, greedy search, exhaustive search, and reinforcement-learning search to optimize the same Tanimoto-guided admission reward.

\subsection{Search Strategies}
\label{app:search-strategies}

\textbf{ChemBack} selects triggers by optimizing a discrete reward over motif-anchor candidates. The reward combines a chemistry-admission term, which requires the candidate to pass \textbf{ChemGuard}, with a target-alignment term based on Tanimoto similarity to the clean target fingerprint bank. Because motif choice, host-anchor choice, motif-anchor choice, sanitization, canonicalization, and graph-string consistency checks are discrete and non-differentiable operations, \textbf{ChemBack} treats trigger selection as a black-box search problem. Importantly, the search procedure is an implementation choice rather than a threat-model assumption.

\textbf{Exhaustive search.}
When the candidate space is small, \textbf{ChemBack} can deterministically enumerate all motif-anchor candidates.
For each candidate, it attempts chemical attachment, applies the corresponding \textbf{ChemGuard} admission rule, computes the Tanimoto score for admitted candidates, and selects the admitted candidate with the highest reward. This strategy provides a strong deterministic reference because it directly optimizes the same objective over the enumerated search space. The main drawback is computational cost, since enumeration becomes increasingly expensive as the motif library grows.

\textbf{Greedy search.}
Greedy search reduces the cost of exhaustive enumeration by evaluating a subset of promising motif-anchor candidates and iteratively improving the selected trigger according to the same reward. At each step, the search keeps candidates that pass chemistry-aware admission and improve target-class Tanimoto similarity.
This strategy is cheaper than exhaustive search while still using the same validity and target-alignment criteria.

\textbf{Reinforcement-learning search.}
RL search parameterizes a proposal distribution over motif and anchor tokens. The policy samples a motif-anchor action, observes whether the resulting molecule passes chemistry-aware admission, and updates the proposal distribution according to the admission and Tanimoto reward. This provides a practical optimizer for larger discrete candidate spaces. RL search does not require gradients through RDKit, Indigo, Open Babel, \textbf{ChemGuard}, molecular fingerprints, the victim model, or the victim training procedure.

We use these three search strategies in the toolkit validation experiment in Appendix~\ref{app:multi-toolkit}. That experiment compares RDKit, Indigo, and Open Babel under the same search variants.
The goal is to disentangle optimizer dependence from toolkit dependence in \textbf{ChemBack}. As shown in Table~\ref{tab:toolkit-validation}, the three search strategies produce comparable high-ASR and fully admitted poisons across toolkit-specific admission rules.

\subsection{Computational Overhead}
\label{app:runtime}

Table~\ref{tab:runtime-overhead} reports runtime and resource overhead for representative attacks across the four main datasets.
\textbf{SearchTime} measures trigger-search or trigger-construction time before victim training. For graph-only baselines, this cost is small because their triggers are fixed or generated by lightweight graph edits. For \textbf{ChemBack}, search time includes chemistry-aware motif-anchor search and toolkit-level validity checks.
\textbf{FullTime} includes trigger search, poisoned-set construction, and victim training under the same experimental environment.
We also report peak GPU memory and average CPU utilization to show that the additional cost mainly comes from discrete chemical search rather than GPU memory growth. 

The results show three trends. First, the graph-only baselines have negligible trigger-search overhead, but this is partly because they do not solve the harder chemistry-admissible trigger construction problem. Second, \textbf{ChemBack-RL} adds only modest overhead over graph-only baselines while maintaining chemistry-admissible poisons.
Third, exhaustive search is substantially more expensive because it enumerates many motif-anchor candidates, while greedy search provides an intermediate cost-accuracy trade-off. These results support the use of RL search as the default optimizer, while confirming that \textbf{ChemBack}'s core mechanism is not tied to a single optimizer.

\begin{table}[t]
\centering
\caption{
Runtime and resource overhead across the four main datasets.
\textbf{SearchTime} is trigger-search or trigger-construction time.
\textbf{FullTime} includes search, poison construction, and victim training.
Values are per seed under the same profiling setup.
}
\label{tab:runtime-overhead}
\footnotesize
\setlength{\tabcolsep}{4.5pt}
\begin{tabular}{cccccc}
\toprule
\textbf{Dataset} & \textbf{Method} &
\textbf{SearchTime (s)} &
\textbf{FullTime (s)} &
\textbf{PeakGPU (MB)} &
\textbf{CPUAvgUtil (\%)} \\
\midrule
\multirow{7}{*}{BBBP}
& GTA                 & 0.02 & 10.72 & 31.93 & 1.53 \\
& MotifBackdoor       & 0.02 & 9.81  & 31.93 & 1.53 \\
& UGBA                & 0.02 & 9.85  & 31.93 & 1.54 \\
& DPGBA               & 0.02 & 9.01  & 31.93 & 1.55 \\
& \textbf{ChemBack}-RL         & 0.03 & 13.34 & 31.93 & 1.56 \\
& \textbf{ChemBack}-Exhaustive & 0.32 & 60.78 & 31.93 & 1.56 \\
& \textbf{ChemBack}-Greedy     & 0.15 & 33.10 & 31.93 & 1.56 \\
\midrule
\multirow{7}{*}{BACE}
& GTA                 & 0.02 & 8.54  & 37.84 & 1.54 \\
& MotifBackdoor       & 0.02 & 7.91  & 37.84 & 1.54 \\
& UGBA                & 0.02 & 8.28  & 37.84 & 1.54 \\
& DPGBA               & 0.02 & 8.40  & 37.84 & 1.54 \\
& \textbf{ChemBack}-RL         & 0.04 & 15.09 & 37.84 & 1.55 \\
& \textbf{ChemBack}-Exhaustive & 0.67 & 91.18 & 37.84 & 1.56 \\
& \textbf{ChemBack}-Greedy     & 0.31 & 48.44 & 37.84 & 1.56 \\
\midrule
\multirow{7}{*}{SIDER}
& GTA                 & 0.02 & 9.34  & 35.71 & 1.54 \\
& MotifBackdoor       & 0.02 & 8.87  & 35.71 & 1.54 \\
& UGBA                & 0.02 & 9.12  & 35.71 & 1.55 \\
& DPGBA               & 0.02 & 9.26  & 35.71 & 1.55 \\
& \textbf{ChemBack}-RL         & 0.04 & 14.62 & 35.71 & 1.56 \\
& \textbf{ChemBack}-Exhaustive & 0.54 & 88.73 & 35.71 & 1.57 \\
& \textbf{ChemBack}-Greedy     & 0.26 & 45.86 & 35.71 & 1.57 \\
\midrule
\multirow{7}{*}{Tox21}
& GTA                 & 0.03 & 34.68 & 48.62 & 1.62 \\
& MotifBackdoor       & 0.03 & 32.94 & 48.62 & 1.62 \\
& UGBA                & 0.03 & 33.71 & 48.62 & 1.63 \\
& DPGBA               & 0.03 & 35.18 & 48.62 & 1.63 \\
& \textbf{ChemBack}-RL         & 0.08 & 42.76 & 48.62 & 1.65 \\
& \textbf{ChemBack}-Exhaustive & 1.84 & 128.64 & 48.62 & 1.66 \\
& \textbf{ChemBack}-Greedy     & 0.92 & 78.37  & 48.62 & 1.66 \\
\bottomrule
\end{tabular}
\end{table}

\section{Sensitivity Analyses}
\label{app:sensitivity}

\subsection{Sensitivity to Poison Rate}
\label{app:sensitivity:poison:rate}

Figure~\ref{fig:sens-poison} reports CA and ASR as a function of poison rate $\alpha\in\{1,5,10\}\%$. Across datasets, ASR increases with the poisoning budget and then begins to saturate, while CA remains relatively stable. This suggests that increasing the poison budget primarily strengthens the backdoor objective rather than causing large clean-utility degradation.

\begin{figure}[ht]
  \centering
  \begin{subfigure}[t]{0.49\linewidth}
    \centering
    \includegraphics[width=0.8\linewidth]{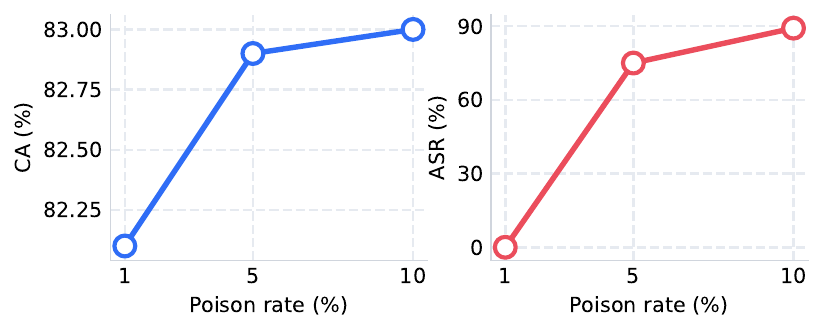}
    \caption{BBBP.}
  \end{subfigure}
  \vspace{0.15cm}
  \begin{subfigure}[t]{0.49\linewidth}
    \centering
    \includegraphics[width=0.8\linewidth]{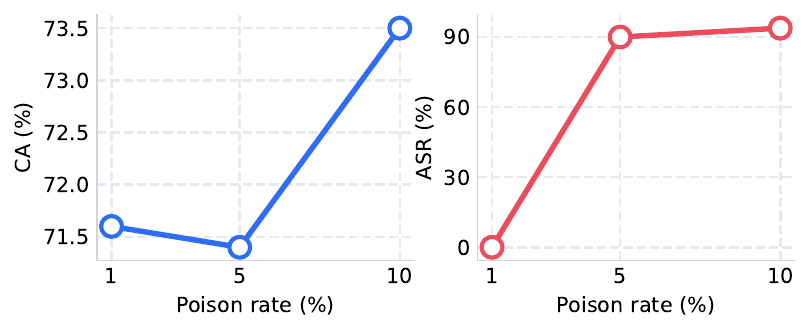}
    \caption{BACE.}
  \end{subfigure}
  \vspace{0.15cm}
  \begin{subfigure}[t]{0.49\linewidth}
    \centering
    \includegraphics[width=0.8\linewidth]{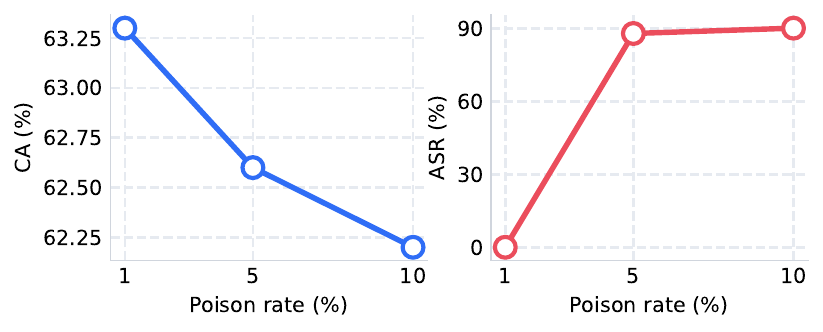}
    \caption{SIDER.}
  \end{subfigure}
  \vspace{0.15cm}
  \begin{subfigure}[t]{0.49\linewidth}
    \centering
    \includegraphics[width=0.8\linewidth]{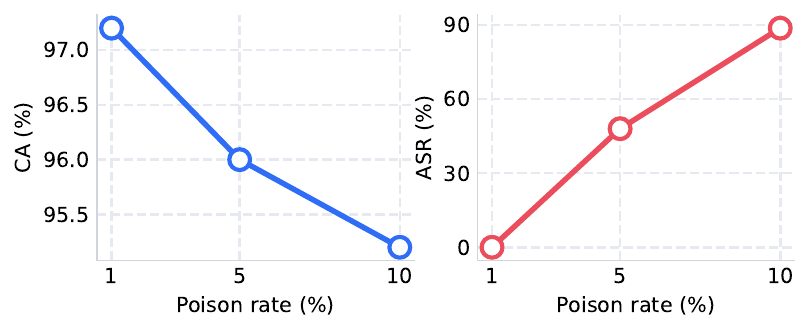}
    \caption{Tox21.}
  \end{subfigure}
  \caption{Sensitivity to poison rate $\alpha\in\{1,5,10\}\%$.
  For each dataset, the left plot shows CA on clean test molecules, and the right plot shows ASR on triggered non-targets.
  Curves are mean$\pm$std over 5 seeds.}
  \label{fig:sens-poison}
\end{figure}

\subsection{Sensitivity to Tanimoto Weight}
\label{app:sensitivity:tanimoto}

We vary the Tanimoto reward weight $\lambda_{\mathrm{Tan}}$ to study the trade-off between target-class structural alignment and chemistry-aware admission during trigger search. When $\lambda_{\mathrm{Tan}}=0$, \textbf{ChemBack} still enforces chemically feasible attachment, but it does not prefer candidates that are structurally close to the clean target class. As $\lambda_{\mathrm{Tan}}$ increases, the selected candidates obtain higher Tanimoto similarity to the clean target fingerprint bank, indicating stronger target-neighborhood alignment.

Figure~\ref{fig:sens-tanimoto} shows that moderate Tanimoto guidance improves structural alignment without destabilizing the attack.
Across datasets, ASR remains stable within a narrow range, while Tanimoto similarity increases and saturates near high-similarity candidates. However, very large $\lambda_{\mathrm{Tan}}$ can slightly reduce the candidate poison-validity rate because the search becomes overly focused on target similarity and may prefer motif-anchor combinations that are harder to realize as valid molecules.
This suggests a practical trade-off in setting $\lambda_{\mathrm{Tan}}$, which should be large enough to select target-aligned poisons without overwhelming attachment feasibility.

In the final poisoning pipeline, \textbf{ChemBack} submits only \textbf{ChemGuard}-admissible poisoned records, so its final EPR remains $100.00\%$ in Table~\ref{tab:main-10}. The sensitivity result instead shows how the internal candidate search behaves before final filtering. Overall, the default setting $\lambda_{\mathrm{Tan}}=1.0$ balances attack effectiveness, chemical admissibility, and target alignment, yielding high ASR, high EPR, and strong Tanimoto similarity to the clean target class.
\begin{figure*}[t]
\centering
\begin{subfigure}[t]{0.49\textwidth}
    \centering
    \includegraphics[width=\linewidth]{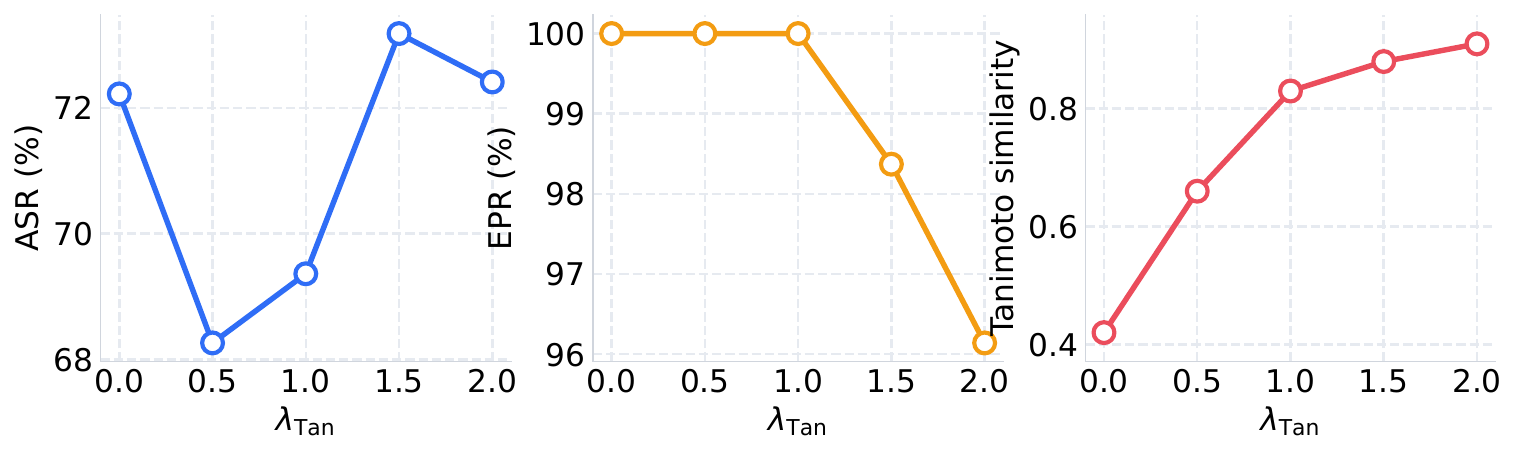}
    \caption{BBBP.}
\end{subfigure}
\hfill
\begin{subfigure}[t]{0.49\textwidth}
    \centering
    \includegraphics[width=\linewidth]{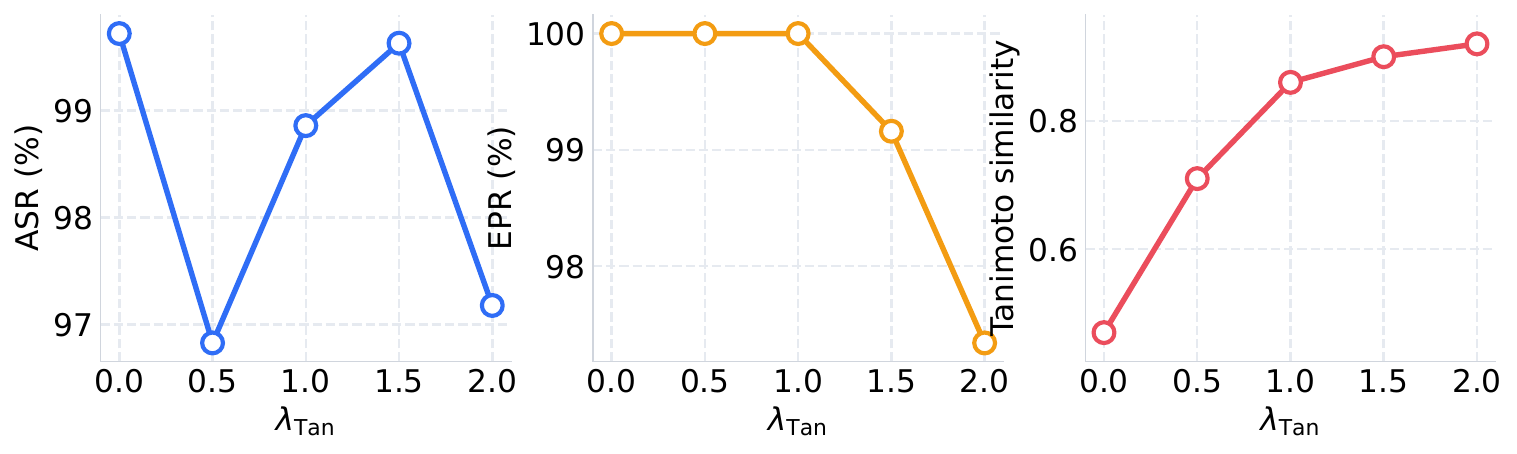}
    \caption{BACE.}
\end{subfigure}
\hfill
\begin{subfigure}[t]{0.49\textwidth}
    \centering
    \includegraphics[width=\linewidth]{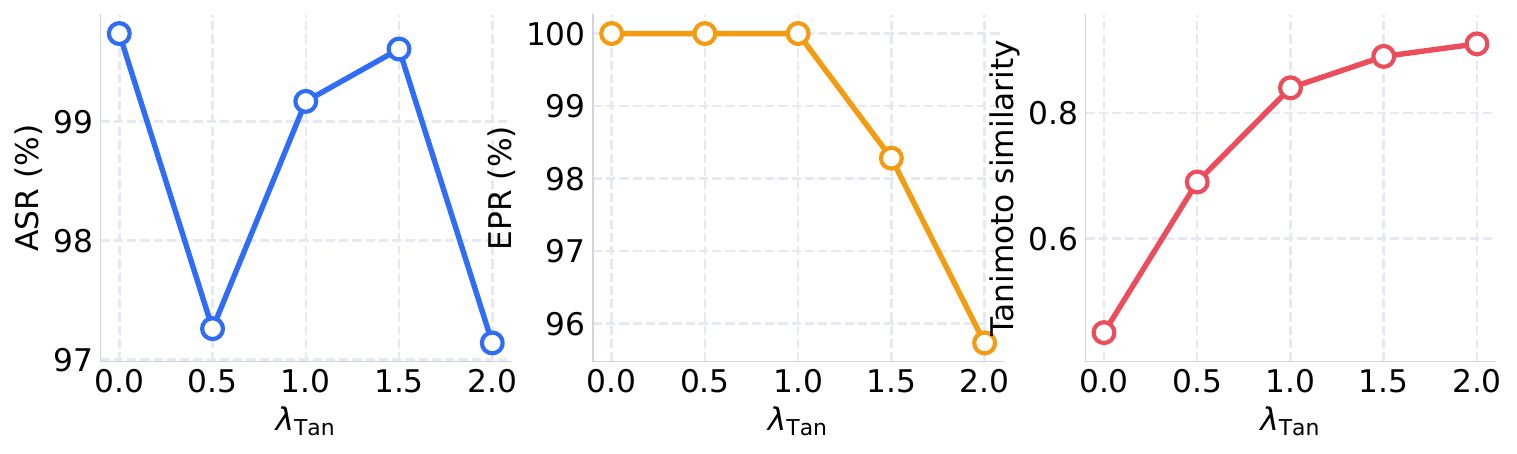}
    \caption{SIDER.}
\end{subfigure}
\hfill
\begin{subfigure}[t]{0.49\textwidth}
    \centering
    \includegraphics[width=\linewidth]{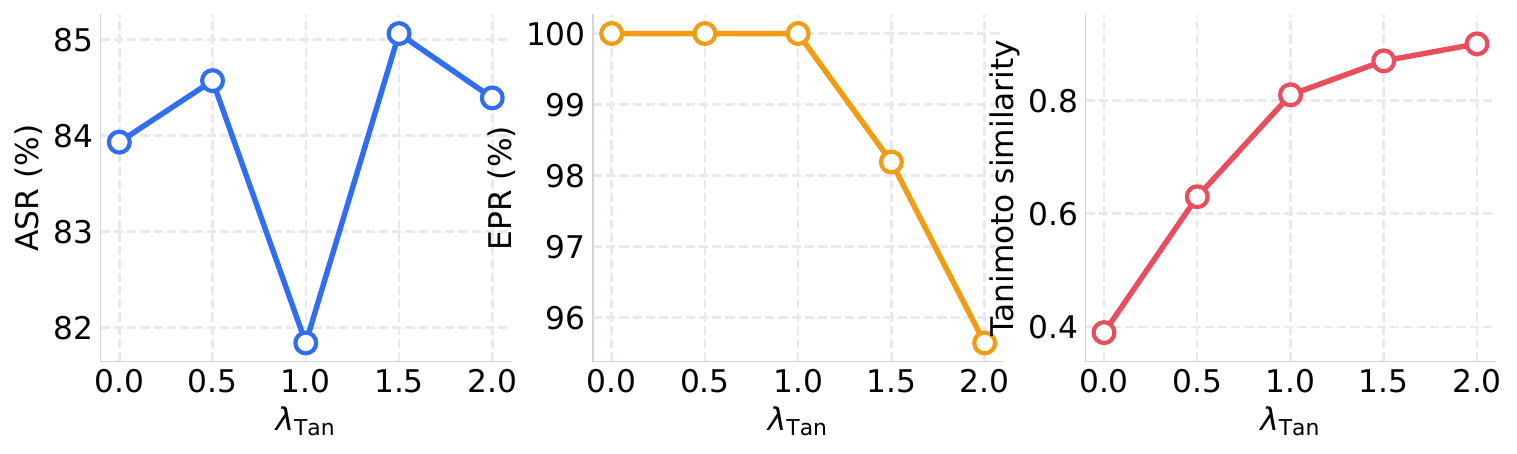}
    \caption{Tox21.}
\end{subfigure}
\caption{Sensitivity to the Tanimoto reward weight $\lambda_{\mathrm{Tan}}$. Each panel reports ASR, EPR, and Tanimoto similarity to the clean target fingerprint bank. Increasing $\lambda_{\mathrm{Tan}}$ improves target-class structural similarity and saturates near high-similarity candidates, while overly large values can slightly reduce candidate validity by over-prioritizing similarity over attachment feasibility.}
\label{fig:sens-tanimoto}
\end{figure*}

\section{Post-hoc Structural and Representation Diagnostics}
\label{app:structural-diagnostics}

This section provides post-hoc diagnostics for \textbf{ChemBack} poisons after trigger generation. These diagnostics are not used by the attacker and are not part of the \textbf{ChemBack} trigger-selection algorithm. They are used only to understand why model-free Tanimoto similarity can produce effective target-aligned poisons and to assess whether the generated poisons are extreme outliers under standard chemistry-facing descriptors.

\subsection{Tanimoto Similarity and Clean-model Representation Proximity}
\label{app:tanimoto-md2}

\textbf{ChemBack} selects admitted candidates using fingerprint-based Tanimoto similarity to clean target-class molecules.
This selection rule is fully model-free and does not require a clean victim model, a proxy GNN, gradients, learned embeddings, or training-code access. To analyze why this structural signal is effective, we train a clean reference model only after trigger generation and compute the target-class Mahalanobis distance, MD$^2$, in the clean model's embedding space.

The main text reports the Tanimoto-MD$^2$ diagnostic in Figure~\ref{fig:tanimoto-md2}. Across datasets, higher Tanimoto similarity to the clean target-class fingerprint bank aligns with lower clean-model target-class MD$^2$. This suggests that the chemistry-native structural neighborhood used by \textbf{ChemBack} is aligned with the learned target region in the victim representation space. Importantly, the clean-model analysis is used only post hoc. \textbf{ChemBack} itself relies only on molecular structures, target labels, fingerprints, and chemistry-aware validity checks.

\subsection{Chemistry-facing Descriptor Balance}
\label{app:descriptor-diagnostics}

Representation-space proximity alone is not a complete stealth guarantee. We therefore further compare \textbf{ChemBack} poisons with clean target-class molecules using standard chemistry-facing descriptors. The evaluated descriptors include physicochemical and structural properties commonly used in molecular screening and drug-likeness analysis, including logP, hydrogen-bond donors and acceptors, rotatable bonds, ring counts, fraction Csp$^3$, QED, formal charge, and stereocenters~\cite{lipinski1997experimental,veber2002molecular,bickerton2012quantifying}.

For each descriptor $d$, we compute the standardized mean difference
\begin{equation}
\mathrm{SMD}(d)
=
\frac{
\mu_p(d)-\mu_c(d)
}{
\sqrt{
\frac{1}{2}
\left(
\sigma_p^2(d)+\sigma_c^2(d)
\right)
}
},
\label{eq:smd}
\end{equation}
where $\mu_p,\sigma_p$ are the mean and standard deviation of the descriptor among \textbf{ChemBack} poisons, and $\mu_c,\sigma_c$ are computed over clean target-class molecules. We report $|\mathrm{SMD}|$ as a scale-normalized descriptor-balance statistic. Small $|\mathrm{SMD}|$ indicates that the poison and clean target-class descriptor distributions have similar central tendency relative to their natural variability~\cite{austin2009balance}.

We also apply a two one-sided equivalence test (TOST) with a pre-specified standardized equivalence bound of $\pm 0.2$~\cite{schuirmann1987comparison,lakens2017equivalence}.
A descriptor is marked as \emph{Pass} when the equivalence interval lies within this bound. This should be interpreted as descriptive statistical evidence of descriptor alignment under the evaluated properties, not as a universal certificate of stealth.

\begin{table*}[t]
\centering
\caption{
Post-hoc descriptor balance between clean target-class molecules and \textbf{ChemBack} poisons on Tox21. Abs SMD denotes the absolute standardized mean difference between clean and poison descriptor distributions. TOST reports whether the two one-sided equivalence test falls within the pre-specified $\pm0.2$ standardized bound.
These diagnostics are computed only after trigger generation and are not used by \textbf{ChemBack}.
}
\label{tab:descriptor-balance}
\footnotesize
\setlength{\tabcolsep}{4pt}
\begin{tabular}{lccccc}
\toprule
\textbf{Descriptor} &
\textbf{Clean mean$\pm$std} &
\textbf{Poison mean$\pm$std} &
\textbf{$\Delta$} &
\textbf{Abs SMD $\downarrow$} &
\textbf{TOST $\pm0.2$} \\
\midrule
logP           & 2.41$\pm$2.18 & 2.57$\pm$2.22 & 0.16  & 0.07 & Pass \\
HBD            & 1.12$\pm$1.31 & 1.20$\pm$1.28 & 0.08  & 0.06 & Pass \\
HBA            & 4.38$\pm$2.82 & 4.71$\pm$2.96 & 0.33  & 0.11 & Pass \\
RotB           & 4.96$\pm$3.74 & 5.42$\pm$3.81 & 0.46  & 0.12 & Pass \\
AromaticRings  & 1.86$\pm$1.18 & 2.05$\pm$1.21 & 0.19  & 0.16 & Pass \\
AliphaticRings & 0.78$\pm$1.02 & 0.89$\pm$1.04 & 0.11  & 0.11 & Pass \\
FractionCSP3   & 0.34$\pm$0.25 & 0.32$\pm$0.24 & -0.02 & 0.08 & Pass \\
QED            & 0.55$\pm$0.20 & 0.53$\pm$0.20 & -0.02 & 0.10 & Pass \\
FormalCharge   & 0.01$\pm$0.14 & 0.01$\pm$0.13 & 0.00  & 0.00 & Pass \\
StereoCenters  & 0.72 $\pm $1.48 & 0.89$\pm$1.51 & 0.17  & 0.11 & Pass \\
\bottomrule
\end{tabular}
\end{table*}

Table~\ref{tab:descriptor-balance} shows that \textbf{ChemBack} poisons remain close to clean target-class molecules under the evaluated descriptor distributions. All descriptors have $|\mathrm{SMD}| \leq 0.16$, and all TOST checks pass under the $\pm0.2$ standardized equivalence bound. Together with the Tanimoto-MD$^2$ diagnostic in the main text, these results support the claim that \textbf{ChemBack} poisons are not extreme outliers under both learned-representation and chemistry-facing diagnostics. At the same time, we do not claim universal stealth against all molecular filters, synthetic-accessibility checks, medicinal-chemistry rules, or human expert inspection.

\subsection{Embedding Visualizations}
\label{app:embedding}

We provide qualitative embedding visualizations to complement the quantitative diagnostics in the main text. These plots are strictly post-hoc analyses. The clean reference model used to compute embeddings and MD$^2$ is never used by \textbf{ChemBack} for trigger selection. The purpose is twofold. First, we visualize why representative graph-only baselines become unreliable under chemistry-aware admission. Second, we visualize why \textbf{ChemBack} remains effective after \textbf{ChemGuard}. Its admitted poisons are not only chemically valid but also close to clean target-class regions under the learned representation.

\paragraph{Baseline failure modes.}
Figure~\ref{fig:app-baseline-embedding-bbbp} visualizes attempted poisons from representative graph backdoors on BBBP. The plot includes clean target molecules, \textbf{ChemGuard}-admitted poisons, and rejected attempted poisons. This distinction matters because rejected points are shown only to diagnose why raw abstract-graph evaluation can overestimate attack success. They do not enter the operational training pipeline. For graph-only baselines such as GTA and MB, many attempted poisons are either rejected or placed in sparse regions away from the dense clean target cluster. UGBA and DPGBA can improve representation-space alignment relative to simpler graph edits, but they still do not guarantee chemistry-aware admission. This supports the main claim that chemical validity and target alignment are both necessary for operational molecular backdoors.

\begin{figure}[t]
  \centering
  \includegraphics[width=0.36\linewidth]{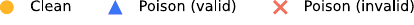}

  \vspace{1mm}

  \begin{subfigure}[t]{0.24\linewidth}
    \centering
    \includegraphics[width=\linewidth]{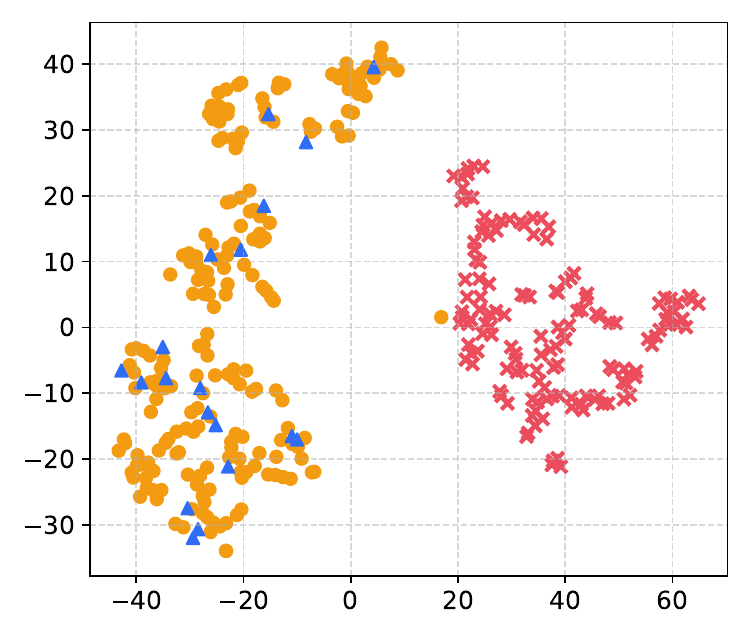}
    \caption{GTA.}
  \end{subfigure}
  \hfill
  \begin{subfigure}[t]{0.24\linewidth}
    \centering
    \includegraphics[width=\linewidth]{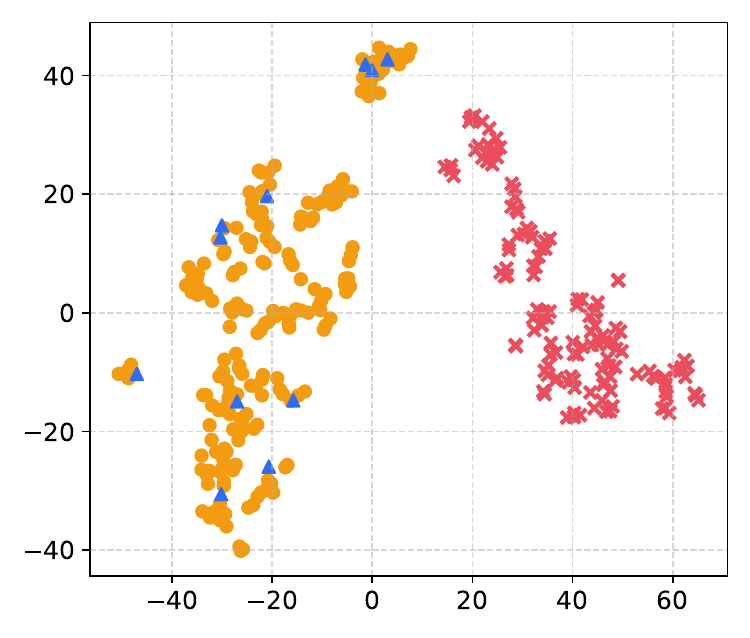}
    \caption{MB.}
  \end{subfigure}
  \hfill
  \begin{subfigure}[t]{0.24\linewidth}
    \centering
    \includegraphics[width=\linewidth]{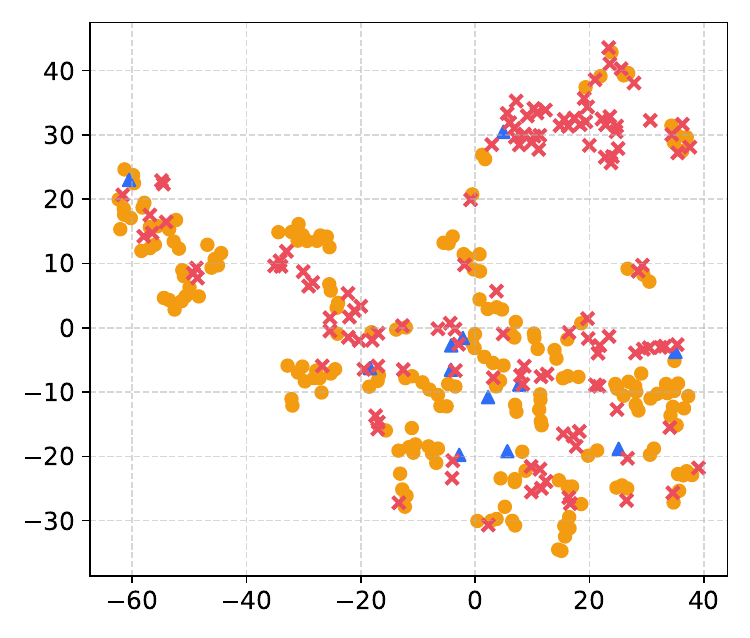}
    \caption{UGBA.}
  \end{subfigure}
  \hfill
  \begin{subfigure}[t]{0.24\linewidth}
    \centering
    \includegraphics[width=\linewidth]{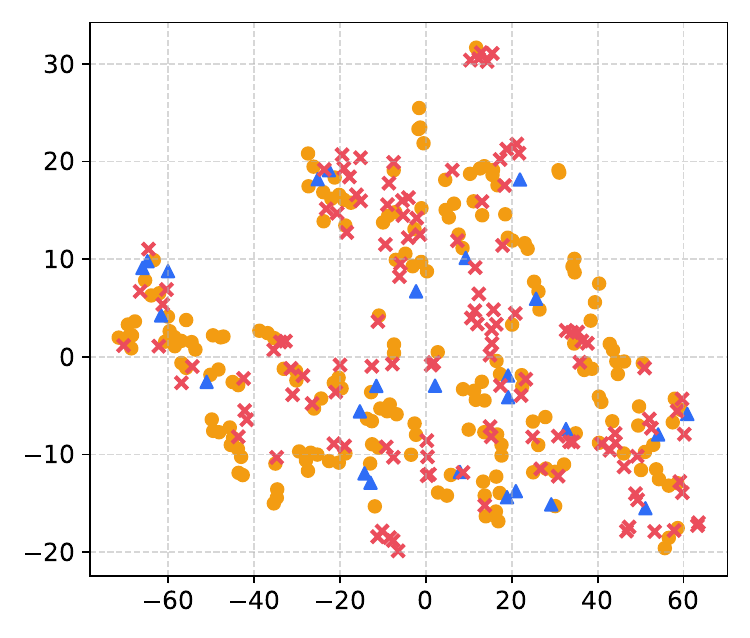}
    \caption{DPGBA.}
  \end{subfigure}

  \caption{
  Post-hoc embedding diagnostics for representative graph backdoors on BBBP. Clean target molecules, \textbf{ChemGuard}-admitted poisons, and rejected attempted poisons are shown in the learned representation space. Rejected poisons are plotted only for diagnosis and do not enter the operational \textbf{ChemGuard} pipeline. Graph-only attacks often produce rejected or off-manifold attempts, explaining why their raw ASR can be inflated relative to chemistry-aware evaluation.
  }
  \label{fig:app-baseline-embedding-bbbp}
\end{figure}

\paragraph{\textbf{ChemBack} admitted poisons.}
Figure~\ref{fig:app-embedding-chemb-bace-tox21} visualizes \textbf{ChemBack} on BACE and Tox21. Unlike the baseline diagnostic above, all plotted \textbf{ChemBack} poisons are \textbf{ChemGuard}-admissible. The left panels overlay clean target molecules and admitted poisons in the learned embedding space. The right panels plot target probability against clean-model MD$^2$. Although MD$^2$ is not used during trigger generation, the post-hoc pattern shows that Tanimoto-selected poisons tend to lie near clean target regions and receive high target probability. Together with the Tanimoto--MD$^2$ relation in the main text and the descriptor-balance results in Appendix~\ref{app:descriptor-diagnostics}, these visualizations support the view that \textbf{ChemBack} produces poisons that are chemically admissible and less outlying under the diagnostics considered in this work.

\begin{figure}[t]
  \centering
  \includegraphics[width=0.22\linewidth]{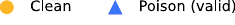}

  \vspace{1mm}

  \begin{subfigure}[t]{0.46\linewidth}
    \centering
    \includegraphics[width=\linewidth]{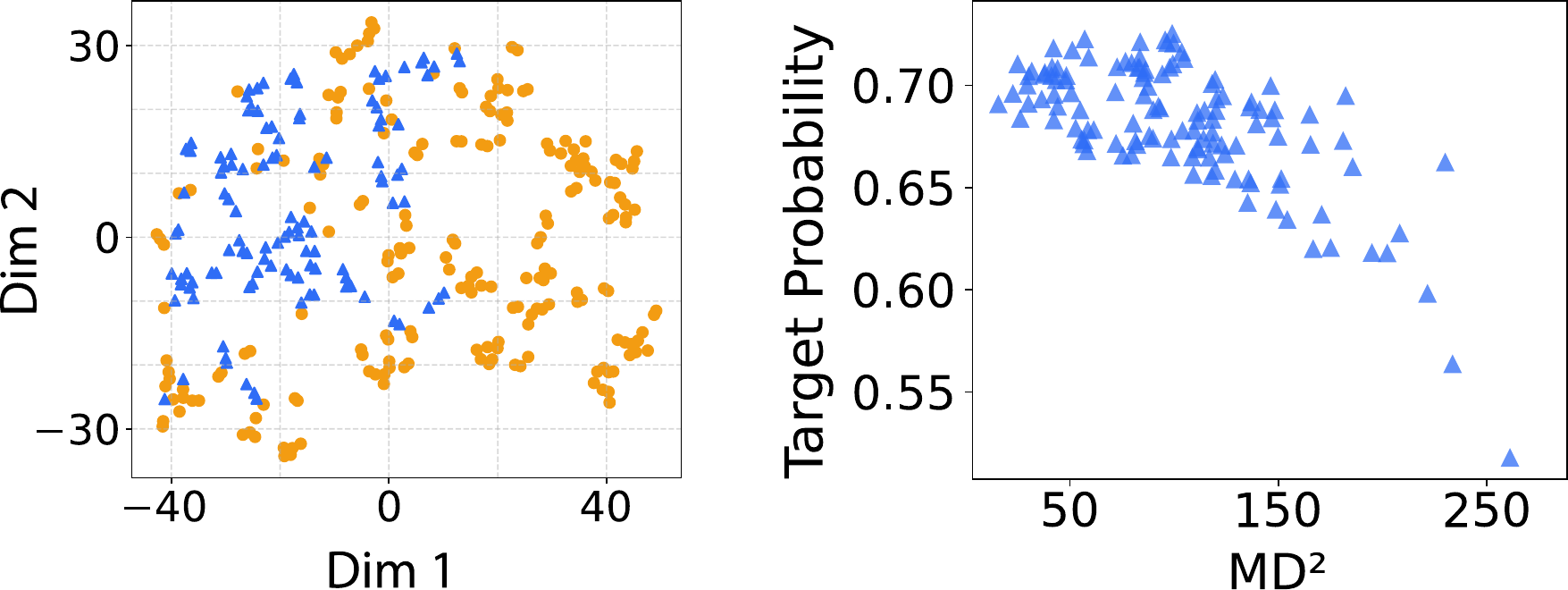}
    \caption{BACE.}
  \end{subfigure}
  \hspace{0.2cm}
  \begin{subfigure}[t]{0.46\linewidth}
    \centering
    \includegraphics[width=\linewidth]{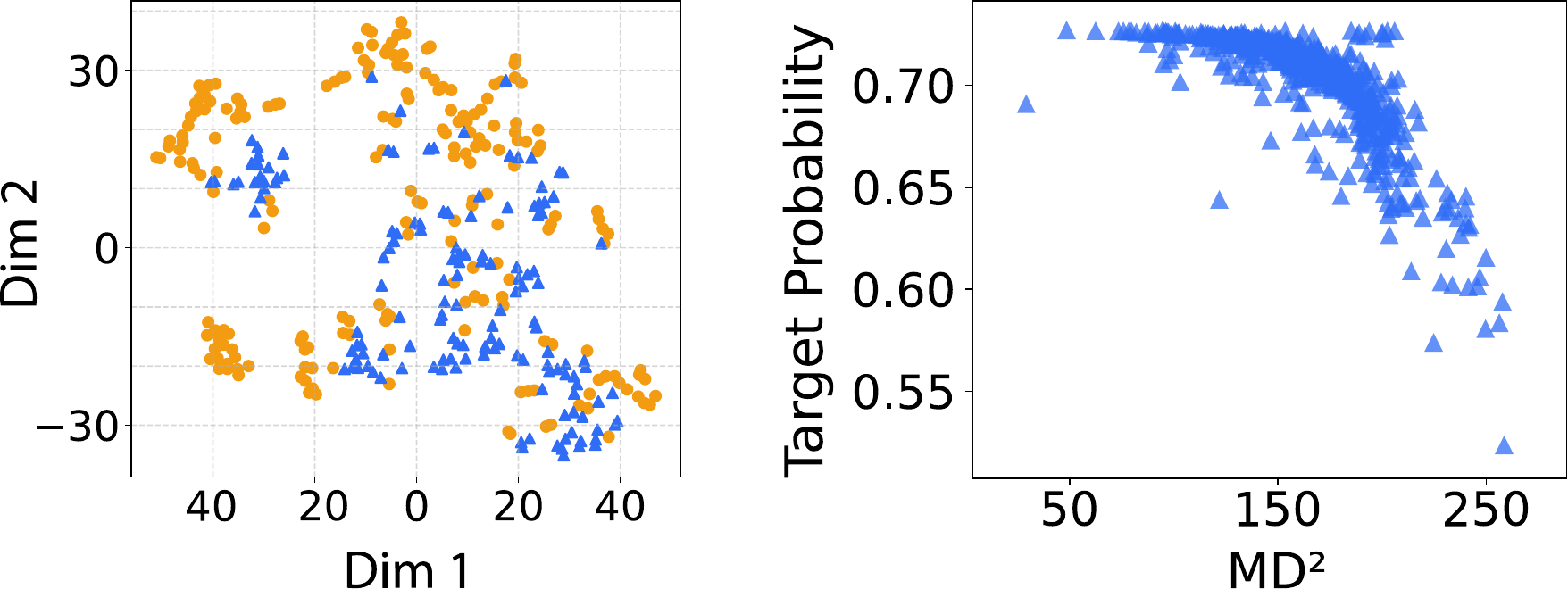}
    \caption{Tox21.}
  \end{subfigure}

  \caption{
  Post-hoc embedding diagnostics for \textbf{ChemBack} on BACE and Tox21. Left panels overlay clean target molecules and \textbf{ChemGuard}-admissible poisons. Right panels plot target probability against clean-model MD$^2$. The clean model is used only for analysis; \textbf{ChemBack} selects triggers using model-free Tanimoto similarity.
  }
  \label{fig:app-embedding-chemb-bace-tox21}
\end{figure}
\section{Robustness Across Architectures}
\label{app:arch-robustness}

To verify that our conclusions are not an artifact of a particular victim backbone, we repeat the poisoning-training-evaluation pipeline under multiple GNN architectures. We consider GCN, GIN, GraphSAGE, GAT, and MPNN. For each architecture, we keep the dataset split, poison rate, target label, training hyperparameters, and operational \textbf{ChemGuard} evaluation protocol identical to the main experimental setup. Invalid trigger realizations are rejected by \textbf{ChemGuard} and counted as attack failures in ASR.

This experiment evaluates architecture robustness, not surrogate-model transfer. \textbf{ChemBack} is model-free. It does not use a victim model, clean model snapshot, proxy GNN, gradients, or learned embeddings during trigger synthesis. Instead, the same chemistry-aware, Tanimoto-selected trigger construction is evaluated across different victim architectures. Table~\ref{tab:arch-robustness} reports CA and ASR across backbones. The GCN column matches the main comparison in Table~\ref{tab:main-10}. Overall, \textbf{ChemBack} maintains high ASR across architectures while keeping CA close to the corresponding no-attack reference, suggesting that its effectiveness comes from chemistry-admissible and target-aligned poison construction rather than a backbone-specific artifact.

\begin{table*}[t]
  \centering
  \caption{
  Robustness across GNN architectures under \textbf{ChemGuard} at $\alpha=10\%$.
  CA is measured on clean test molecules, and ASR is measured on triggered non-target molecules under operational \textbf{ChemGuard} evaluation.
  Invalid trigger realizations are counted as attack failures.
  Values are mean$\pm$std over 5 seeds.
  }
  \label{tab:arch-robustness}
  \scriptsize
  \setlength{\tabcolsep}{2.4pt}
  \resizebox{\textwidth}{!}{
  \begin{tabular}{llcccccccccc}
\toprule
\multirow{2}{*}{\textbf{Dataset}} &
\multirow{2}{*}{\textbf{Method}} &
\multicolumn{2}{c}{\textbf{GCN}} &
\multicolumn{2}{c}{\textbf{GIN}} &
\multicolumn{2}{c}{\textbf{GraphSAGE}} &
\multicolumn{2}{c}{\textbf{GAT}} &
\multicolumn{2}{c}{\textbf{MPNN}} \\
\cmidrule(lr){3-4}
\cmidrule(lr){5-6}
\cmidrule(lr){7-8}
\cmidrule(lr){9-10}
\cmidrule(lr){11-12}
& &
\textbf{CA} & \textbf{ASR} &
\textbf{CA} & \textbf{ASR} &
\textbf{CA} & \textbf{ASR} &
\textbf{CA} & \textbf{ASR} &
\textbf{CA} & \textbf{ASR} \\
\midrule

\multirow{6}{*}{\textbf{BBBP}}
& No-Attack
& \mstd{86.34}{0.43} & -
& \mstd{86.24}{0.53} & -
& \mstd{85.94}{0.63} & -
& \mstd{86.06}{0.43} & -
& \mstd{86.73}{0.43} & - \\
& GTA
& \mstd{84.23}{0.46} & \mstd{54.47}{0.56}
& \mstd{85.14}{1.34} & \mstd{54.36}{10.04}
& \mstd{85.74}{1.03} & \mstd{58.34}{2.23}
& \mstd{85.54}{0.53} & \mstd{52.76}{9.93}
& \mstd{85.03}{3.04} & \mstd{54.06}{15.13} \\
& MB
& \mstd{84.35}{0.63} & \mstd{46.27}{8.24}
& \mstd{86.04}{0.23} & \mstd{81.36}{2.23}
& \mstd{85.04}{1.74} & \mstd{86.76}{4.23}
& \mstd{86.23}{0.43} & \mstd{80.74}{3.56}
& \mstd{86.06}{0.73} & \mstd{81.36}{3.24} \\
& UGBA
& \mstd{83.94}{0.34} & \mstd{47.83}{3.86}
& \mstd{85.64}{0.63} & \mstd{72.74}{4.53}
& \mstd{85.66}{0.64} & \mstd{78.34}{5.54}
& \mstd{85.94}{0.53} & \mstd{77.76}{1.54}
& \mstd{85.93}{1.24} & \mstd{70.06}{6.63} \\
& DPGBA
& \mstd{84.27}{0.84} & \mstd{52.86}{3.36}
& \mstd{86.84}{0.43} & \mstd{68.04}{2.03}
& \mstd{85.34}{1.13} & \mstd{73.46}{3.94}
& \mstd{85.64}{0.53} & \mstd{74.06}{2.04}
& \mstd{85.66}{0.73} & \mstd{63.76}{10.03} \\
& \textbf{ChemBack}
& \mstd{81.08}{0.52} & \mstd{69.36}{1.14}
& \mstd{85.74}{0.93} & \mstd{81.94}{12.23}
& \mstd{85.86}{1.13} & \mstd{86.36}{3.94}
& \mstd{85.34}{0.23} & \mstd{87.34}{14.13}
& \mstd{86.04}{1.13} & \mstd{84.94}{6.03} \\

\midrule
\multirow{6}{*}{\textbf{BACE}}
& No-Attack
& \mstd{71.64}{0.53} & -
& \mstd{71.14}{1.23} & -
& \mstd{71.16}{1.94} & -
& \mstd{70.43}{2.24} & -
& \mstd{72.74}{0.73} & - \\
& GTA
& \mstd{71.32}{2.47} & \mstd{27.46}{2.14}
& \mstd{70.23}{2.24} & \mstd{34.26}{12.14}
& \mstd{78.14}{1.74} & \mstd{39.54}{16.34}
& \mstd{72.83}{0.73} & \mstd{37.46}{9.74}
& \mstd{73.84}{2.04} & \mstd{27.26}{1.14} \\
& MB
& \mstd{71.06}{1.94} & \mstd{13.38}{8.74}
& \mstd{70.84}{2.83} & \mstd{9.76}{5.34}
& \mstd{70.86}{3.64} & \mstd{9.86}{0.64}
& \mstd{72.84}{3.43} & \mstd{9.24}{0.43}
& \mstd{72.26}{2.23} & \mstd{9.64}{1.34} \\
& UGBA
& \mstd{70.74}{0.34} & \mstd{15.37}{3.34}
& \mstd{71.34}{1.03} & \mstd{95.56}{2.64}
& \mstd{71.16}{3.34} & \mstd{96.16}{1.84}
& \mstd{70.34}{4.14} & \mstd{95.96}{0.94}
& \mstd{69.34}{4.94} & \mstd{93.26}{2.84} \\
& DPGBA
& \mstd{70.24}{1.24} & \mstd{16.47}{4.54}
& \mstd{71.73}{1.14} & \mstd{51.64}{1.74}
& \mstd{75.24}{0.73} & \mstd{53.16}{2.24}
& \mstd{71.64}{1.43} & \mstd{43.64}{5.34}
& \mstd{72.46}{1.03} & \mstd{51.46}{7.74} \\
& \textbf{ChemBack}
& \mstd{70.63}{0.72} & \mstd{98.86}{0.43}
& \mstd{71.13}{1.03} & \mstd{96.56}{3.43}
& \mstd{74.13}{2.43} & \mstd{97.76}{2.04}
& \mstd{73.74}{0.73} & \mstd{97.94}{2.63}
& \mstd{70.54}{3.63} & \mstd{96.16}{2.93} \\

\midrule
\multirow{6}{*}{\textbf{SIDER}}
& No-Attack
& \mstd{63.34}{0.74} & -
& \mstd{63.43}{2.23} & -
& \mstd{63.24}{1.74} & -
& \mstd{63.46}{1.93} & -
& \mstd{63.34}{1.94} & - \\
& GTA
& \mstd{60.74}{0.54} & \mstd{49.96}{7.84}
& \mstd{61.43}{1.84} & \mstd{58.54}{1.43}
& \mstd{63.13}{3.43} & \mstd{54.26}{12.34}
& \mstd{62.64}{1.23} & \mstd{51.86}{2.73}
& \mstd{63.34}{4.43} & \mstd{50.76}{19.24} \\
& MB
& \mstd{60.07}{1.46} & \mstd{51.17}{3.86}
& \mstd{60.84}{0.63} & \mstd{89.74}{1.13}
& \mstd{65.74}{1.94} & \mstd{87.84}{2.83}
& \mstd{62.24}{1.73} & \mstd{89.46}{0.73}
& \mstd{65.64}{1.13} & \mstd{87.36}{0.73} \\
& UGBA
& \mstd{60.16}{0.94} & \mstd{53.07}{4.36}
& \mstd{64.34}{3.03} & \mstd{86.94}{1.13}
& \mstd{65.76}{1.34} & \mstd{88.24}{2.23}
& \mstd{64.54}{2.63} & \mstd{87.84}{2.24}
& \mstd{63.54}{2.23} & \mstd{87.64}{2.04} \\
& DPGBA
& \mstd{60.84}{0.34} & \mstd{45.27}{4.46}
& \mstd{63.34}{1.53} & \mstd{58.96}{15.54}
& \mstd{60.34}{1.84} & \mstd{73.74}{7.13}
& \mstd{61.74}{3.33} & \mstd{72.84}{7.83}
& \mstd{63.36}{1.43} & \mstd{50.74}{19.23} \\
& \textbf{ChemBack}
& \mstd{60.84}{0.63} & \mstd{99.17}{0.34}
& \mstd{62.43}{1.23} & \mstd{96.43}{3.93}
& \mstd{61.14}{1.13} & \mstd{95.13}{3.13}
& \mstd{62.44}{1.93} & \mstd{94.34}{3.23}
& \mstd{62.14}{0.73} & \mstd{93.54}{5.53} \\

\midrule
\multirow{6}{*}{\textbf{Tox21}}
& No-Attack
& \mstd{97.24}{0.34} & -
& \mstd{96.74}{0.13} & -
& \mstd{96.84}{0.13} & -
& \mstd{96.83}{0.23} & -
& \mstd{96.74}{0.04} & - \\
& GTA
& \mstd{97.14}{0.24} & \mstd{16.73}{2.18}
& \mstd{95.54}{0.73} & \mstd{51.74}{2.23}
& \mstd{95.94}{0.83} & \mstd{50.86}{5.04}
& \mstd{95.24}{0.93} & \mstd{50.94}{2.73}
& \mstd{95.64}{0.73} & \mstd{50.76}{3.74} \\
& MB
& \mstd{97.23}{0.13} & \mstd{24.61}{3.57}
& \mstd{96.43}{0.34} & \mstd{95.04}{0.33}
& \mstd{96.54}{0.13} & \mstd{95.64}{0.13}
& \mstd{96.34}{0.23} & \mstd{95.44}{0.23}
& \mstd{96.24}{0.33} & \mstd{95.14}{0.33} \\
& UGBA
& \mstd{97.26}{0.34} & \mstd{22.84}{2.46}
& \mstd{96.34}{0.33} & \mstd{95.64}{0.23}
& \mstd{96.24}{0.23} & \mstd{95.24}{0.43}
& \mstd{96.44}{0.23} & \mstd{95.64}{0.23}
& \mstd{96.34}{0.33} & \mstd{95.64}{0.23} \\
& DPGBA
& \mstd{97.24}{0.23} & \mstd{18.49}{2.69}
& \mstd{96.24}{0.63} & \mstd{84.34}{4.63}
& \mstd{96.24}{0.83} & \mstd{83.74}{1.23}
& \mstd{96.04}{0.83} & \mstd{84.24}{1.03}
& \mstd{96.43}{0.83} & \mstd{84.04}{1.03} \\
& \textbf{ChemBack}
& \mstd{96.34}{0.43} & \mstd{81.84}{0.84}
& \mstd{96.04}{0.53} & \mstd{91.43}{14.03}
& \mstd{95.94}{0.53} & \mstd{93.94}{9.93}
& \mstd{95.84}{1.03} & \mstd{95.64}{7.13}
& \mstd{95.84}{0.83} & \mstd{96.34}{3.53} \\
\bottomrule
\end{tabular}
}
\end{table*}
\section{Additional Model-level and Post-hoc Defense Results}
\label{app:model-level-defenses}

This appendix provides full CA/ASR results for the defenses discussed in the main text and for two additional model-level robust GNN defenses.
All results are evaluated at $\alpha=10\%$ under the raw pipeline, i.e., after poisoned samples have entered the learning pipeline.
This setting differs intentionally from \textbf{ChemGuard}. \textbf{ChemGuard} performs admission-stage chemistry checks before training, whereas the defenses in this appendix operate after admission.

\paragraph{Full results for main-text defenses.} 
Table~\ref{tab:app-defense-full-main} reports full CA/ASR results for Spectral Signatures, DShield, RGCN, RIGBD, and PGNNCert. The Raw columns match Table~\ref{tab:main-10}. Spectral Signatures has a larger effect on graph-only baselines whose poisons are more separable in representation space, while UGBA and DPGBA are less affected by Spectral Signatures because they explicitly target unnoticeability or distribution preservation. \textbf{ChemBack} is also only mildly affected by Spectral Signatures because its poisons are \textbf{ChemGuard}-admissible and selected using target-class Tanimoto similarity. RIGBD and PGNNCert can reduce ASR more strongly because they operate through robustness-inspired or certified/model-level mechanisms, but they still do not remove \textbf{ChemBack} after poisons enter training.

\begin{table*}[t]
\centering
\caption{
Full defense results for the main-text defenses at $\alpha=10\%$ under the raw pipeline. Values report CA/ASR (\%) as mean$\pm$std over 5 seeds. Raw columns match Table~\ref{tab:main-10}.
}
\label{tab:app-defense-full-main}
\scriptsize
\setlength{\tabcolsep}{2.1pt}
\resizebox{\textwidth}{!}{
\begin{tabular}{llcccccccccccc}
\toprule
\textbf{Dataset} & \textbf{Attack}
& \multicolumn{2}{c}{\textbf{Raw}}
& \multicolumn{2}{c}{\textbf{Spectral}}
& \multicolumn{2}{c}{\textbf{DShield}}
& \multicolumn{2}{c}{\textbf{RGCN}}
& \multicolumn{2}{c}{\textbf{RIGBD}}
& \multicolumn{2}{c}{\textbf{PGNNCert}} \\
\cmidrule(lr){3-4}
\cmidrule(lr){5-6}
\cmidrule(lr){7-8}
\cmidrule(lr){9-10}
\cmidrule(lr){11-12}
\cmidrule(lr){13-14}
& & \textbf{CA} & \textbf{ASR}
& \textbf{CA} & \textbf{ASR}
& \textbf{CA} & \textbf{ASR}
& \textbf{CA} & \textbf{ASR}
& \textbf{CA} & \textbf{ASR}
& \textbf{CA} & \textbf{ASR} \\
\midrule

\multirow{5}{*}{\textbf{BBBP}}
& GTA
& \mstd{79.72}{0.63} & \mstd{73.16}{1.24}
& \mstd{83.47}{0.58} & \mstd{48.93}{1.47}
& \mstd{83.28}{0.54} & \mstd{42.17}{1.63}
& \mstd{82.74}{0.69} & \mstd{45.86}{1.58}
& \mstd{83.06}{0.67} & \mstd{37.42}{1.79}
& \mstd{80.48}{0.76} & \mstd{36.94}{1.56} \\
& MB
& \mstd{80.43}{0.64} & \mstd{72.18}{1.36}
& \mstd{83.61}{0.67} & \mstd{50.26}{4.87}
& \mstd{83.42}{0.63} & \mstd{43.72}{4.59}
& \mstd{82.89}{0.78} & \mstd{46.83}{4.28}
& \mstd{83.18}{0.74} & \mstd{35.61}{4.36}
& \mstd{80.23}{0.82} & \mstd{46.74}{4.91} \\
& UGBA
& \mstd{82.14}{0.53} & \mstd{60.67}{1.16}
& \mstd{83.34}{0.46} & \mstd{59.38}{1.92}
& \mstd{83.17}{0.51} & \mstd{56.21}{2.38}
& \mstd{82.76}{0.58} & \mstd{57.46}{2.14}
& \mstd{82.94}{0.54} & \mstd{34.87}{2.41}
& \mstd{80.86}{0.68} & \mstd{35.24}{2.19} \\
& DPGBA
& \mstd{81.93}{0.65} & \mstd{69.24}{1.13}
& \mstd{83.43}{0.84} & \mstd{68.16}{2.21}
& \mstd{83.28}{0.86} & \mstd{65.47}{2.36}
& \mstd{82.97}{0.91} & \mstd{66.83}{2.17}
& \mstd{83.14}{0.93} & \mstd{48.92}{2.53}
& \mstd{82.16}{0.87} & \mstd{49.31}{2.47} \\
& \textbf{ChemBack}
& \mstd{81.08}{0.52} & \mstd{69.36}{1.14}
& \mstd{80.96}{0.61} & \mstd{68.74}{1.39}
& \mstd{80.73}{0.64} & \mstd{66.58}{1.72}
& \mstd{80.84}{0.58} & \mstd{67.13}{1.46}
& \mstd{80.27}{0.73} & \mstd{61.37}{1.85}
& \mstd{80.91}{0.62} & \mstd{68.28}{1.32} \\

\midrule
\multirow{5}{*}{\textbf{BACE}}
& GTA
& \mstd{67.34}{0.54} & \mstd{47.96}{1.14}
& \mstd{70.43}{0.83} & \mstd{32.76}{2.13}
& \mstd{70.28}{0.94} & \mstd{27.84}{1.96}
& \mstd{70.11}{0.87} & \mstd{29.63}{2.18}
& \mstd{70.24}{1.06} & \mstd{21.38}{1.83}
& \mstd{70.16}{0.92} & \mstd{22.47}{1.94} \\
& MB
& \mstd{70.04}{0.53} & \mstd{82.67}{1.54}
& \mstd{70.23}{0.64} & \mstd{60.21}{2.47}
& \mstd{70.17}{0.72} & \mstd{53.86}{2.18}
& \mstd{70.08}{0.76} & \mstd{55.74}{2.26}
& \mstd{70.13}{0.84} & \mstd{15.42}{4.73}
& \mstd{70.06}{0.83} & \mstd{48.36}{4.97} \\
& UGBA
& \mstd{72.64}{0.63} & \mstd{67.18}{1.24}
& \mstd{70.54}{0.68} & \mstd{65.92}{1.71}
& \mstd{70.36}{0.63} & \mstd{61.47}{1.68}
& \mstd{70.42}{0.71} & \mstd{62.84}{1.76}
& \mstd{70.18}{0.74} & \mstd{28.61}{2.34}
& \mstd{70.24}{0.69} & \mstd{31.86}{2.27} \\
& DPGBA
& \mstd{69.36}{0.54} & \mstd{55.74}{1.23}
& \mstd{69.86}{0.63} & \mstd{54.83}{1.76}
& \mstd{69.74}{0.67} & \mstd{52.13}{1.82}
& \mstd{69.69}{0.71} & \mstd{53.46}{1.67}
& \mstd{69.64}{0.73} & \mstd{31.28}{3.16}
& \mstd{69.78}{0.68} & \mstd{35.74}{3.04} \\
& \textbf{ChemBack}
& \mstd{70.63}{0.72} & \mstd{98.86}{0.43}
& \mstd{70.54}{0.73} & \mstd{97.84}{0.76}
& \mstd{70.42}{0.81} & \mstd{96.37}{1.08}
& \mstd{70.31}{0.77} & \mstd{96.91}{0.94}
& \mstd{69.94}{0.92} & \mstd{92.38}{1.63}
& \mstd{70.18}{0.86} & \mstd{97.16}{0.81} \\

\midrule
\multirow{5}{*}{\textbf{SIDER}}
& GTA
& \mstd{60.83}{0.74} & \mstd{66.47}{1.43}
& \mstd{60.64}{0.78} & \mstd{45.36}{1.57}
& \mstd{60.52}{0.76} & \mstd{40.78}{1.49}
& \mstd{60.39}{0.82} & \mstd{42.19}{1.68}
& \mstd{60.46}{0.86} & \mstd{36.74}{1.53}
& \mstd{60.42}{0.74} & \mstd{38.21}{1.64} \\
& MB
& \mstd{64.73}{0.73} & \mstd{83.28}{1.46}
& \mstd{64.52}{0.86} & \mstd{63.17}{1.89}
& \mstd{64.39}{0.83} & \mstd{58.42}{1.73}
& \mstd{64.17}{0.91} & \mstd{55.83}{2.14}
& \mstd{58.86}{1.62} & \mstd{40.69}{2.85}
& \mstd{60.54}{1.37} & \mstd{41.24}{2.37} \\
& UGBA
& \mstd{59.47}{0.64} & \mstd{86.18}{1.74}
& \mstd{59.34}{0.68} & \mstd{84.97}{1.82}
& \mstd{59.26}{0.67} & \mstd{82.36}{1.69}
& \mstd{59.18}{0.71} & \mstd{82.74}{1.51}
& \mstd{59.04}{0.74} & \mstd{41.28}{3.17}
& \mstd{59.21}{0.63} & \mstd{80.37}{1.58} \\
& DPGBA
& \mstd{62.94}{0.63} & \mstd{49.67}{1.04}
& \mstd{62.82}{0.64} & \mstd{48.72}{1.16}
& \mstd{62.71}{0.66} & \mstd{47.16}{1.09}
& \mstd{62.64}{0.69} & \mstd{48.31}{1.03}
& \mstd{59.86}{0.57} & \mstd{44.26}{3.28}
& \mstd{61.63}{0.72} & \mstd{48.86}{1.12} \\
& \textbf{ChemBack}
& \mstd{60.84}{0.63} & \mstd{99.17}{0.34}
& \mstd{60.72}{0.63} & \mstd{98.52}{0.61}
& \mstd{60.56}{0.72} & \mstd{96.39}{0.82}
& \mstd{60.48}{0.69} & \mstd{96.87}{0.94}
& \mstd{60.13}{0.83} & \mstd{92.74}{1.46}
& \mstd{60.31}{0.78} & \mstd{93.46}{1.21} \\

\midrule
\multirow{5}{*}{\textbf{Tox21}}
& GTA
& \mstd{96.43}{0.43} & \mstd{42.78}{1.29}
& \mstd{96.54}{0.34} & \mstd{30.86}{1.47}
& \mstd{96.52}{0.36} & \mstd{27.49}{1.38}
& \mstd{96.61}{0.32} & \mstd{29.64}{1.42}
& \mstd{96.84}{0.23} & \mstd{23.76}{1.51}
& \mstd{96.78}{0.24} & \mstd{25.83}{1.47} \\
& MB
& \mstd{96.94}{0.34} & \mstd{96.37}{0.63}
& \mstd{96.84}{0.33} & \mstd{92.83}{0.72}
& \mstd{96.72}{0.34} & \mstd{85.67}{0.96}
& \mstd{96.69}{0.32} & \mstd{88.43}{0.87}
& \mstd{96.87}{0.16} & \mstd{61.24}{0.58}
& \mstd{96.81}{0.19} & \mstd{74.68}{0.63} \\
& UGBA
& \mstd{96.53}{0.33} & \mstd{97.06}{0.54}
& \mstd{96.48}{0.32} & \mstd{96.78}{0.57}
& \mstd{96.41}{0.34} & \mstd{93.21}{0.68}
& \mstd{96.39}{0.31} & \mstd{94.16}{0.62}
& \mstd{96.76}{0.33} & \mstd{68.47}{0.76}
& \mstd{96.69}{0.31} & \mstd{91.38}{0.64} \\
& DPGBA
& \mstd{96.56}{0.43} & \mstd{52.76}{0.94}
& \mstd{96.52}{0.42} & \mstd{51.84}{0.87}
& \mstd{96.46}{0.43} & \mstd{50.39}{0.91}
& \mstd{96.42}{0.46} & \mstd{51.18}{0.84}
& \mstd{96.88}{0.26} & \mstd{46.73}{0.58}
& \mstd{96.73}{0.28} & \mstd{50.82}{0.71} \\
& \textbf{ChemBack}
& \mstd{96.34}{0.43} & \mstd{81.84}{0.84}
& \mstd{96.24}{0.42} & \mstd{81.12}{0.91}
& \mstd{96.13}{0.43} & \mstd{79.86}{1.08}
& \mstd{96.08}{0.46} & \mstd{80.34}{0.96}
& \mstd{95.92}{0.51} & \mstd{74.63}{1.46}
& \mstd{96.01}{0.49} & \mstd{78.92}{1.18} \\

\bottomrule
\end{tabular}
}
\end{table*}

\paragraph{Additional robust GNN defenses.}
Table~\ref{tab:app-defense-gnnguard-prognn} reports results for GNNGuard and Pro-GNN. These defenses modify propagation or learn graph structure after samples have already entered training. They can reduce learned backdoor behavior in some settings, but they do not perform chemistry-aware molecular sanitization or graph-string consistency checks. Thus, they are complementary to \textbf{ChemGuard} rather than replacements for it.

\begin{table*}[t]
\centering
\caption{
Additional robust GNN defense results at $\alpha=10\%$ under the raw pipeline. Values report CA/ASR (\%) as mean$\pm$std over 5 seeds.
GNNGuard and Pro-GNN are evaluated as model-level defenses after poisoned records have entered training.
}
\label{tab:app-defense-gnnguard-prognn}
\footnotesize
\setlength{\tabcolsep}{6pt}
\begin{tabular}{cccccccc}
\toprule
\textbf{Dataset} & \textbf{Attack}
& \multicolumn{2}{c}{\textbf{Raw}}
& \multicolumn{2}{c}{\textbf{GNNGuard}}
& \multicolumn{2}{c}{\textbf{Pro-GNN}} \\
\cmidrule(lr){3-4}
\cmidrule(lr){5-6}
\cmidrule(lr){7-8}
& & \textbf{CA} & \textbf{ASR}
& \textbf{CA} & \textbf{ASR}
& \textbf{CA} & \textbf{ASR} \\
\midrule

\multirow{5}{*}{\textbf{BBBP}}
& GTA
& \mstd{79.72}{0.63} & \mstd{73.16}{1.24}
& \mstd{82.86}{0.54} & \mstd{43.72}{1.58}
& \mstd{83.14}{0.61} & \mstd{39.18}{1.64} \\
& MB
& \mstd{80.43}{0.64} & \mstd{72.18}{1.36}
& \mstd{82.94}{0.62} & \mstd{45.86}{4.77}
& \mstd{83.07}{0.67} & \mstd{41.62}{4.83} \\
& UGBA
& \mstd{82.14}{0.53} & \mstd{60.67}{1.16}
& \mstd{82.71}{0.53} & \mstd{55.27}{2.46}
& \mstd{83.02}{0.57} & \mstd{51.84}{2.38} \\
& DPGBA
& \mstd{81.93}{0.65} & \mstd{69.24}{1.13}
& \mstd{82.83}{0.71} & \mstd{64.13}{2.74}
& \mstd{83.11}{0.76} & \mstd{60.32}{2.68} \\
& \textbf{ChemBack}
& \mstd{81.08}{0.52} & \mstd{69.36}{1.14}
& \mstd{80.74}{0.66} & \mstd{66.21}{1.43}
& \mstd{80.52}{0.71} & \mstd{64.77}{1.58} \\

\midrule
\multirow{5}{*}{\textbf{BACE}}
& GTA
& \mstd{67.34}{0.54} & \mstd{47.96}{1.14}
& \mstd{70.64}{0.84} & \mstd{30.83}{2.16}
& \mstd{70.38}{0.91} & \mstd{25.64}{2.24} \\
& MB
& \mstd{70.04}{0.53} & \mstd{82.67}{1.54}
& \mstd{70.21}{0.67} & \mstd{57.42}{2.63}
& \mstd{70.16}{0.72} & \mstd{48.93}{3.18} \\
& UGBA
& \mstd{72.64}{0.63} & \mstd{67.18}{1.24}
& \mstd{70.33}{0.68} & \mstd{64.21}{1.74}
& \mstd{70.24}{0.73} & \mstd{60.47}{1.96} \\
& DPGBA
& \mstd{69.36}{0.54} & \mstd{55.74}{1.23}
& \mstd{69.81}{0.64} & \mstd{53.17}{1.83}
& \mstd{69.62}{0.69} & \mstd{50.86}{1.94} \\
& \textbf{ChemBack}
& \mstd{70.63}{0.72} & \mstd{98.86}{0.43}
& \mstd{70.18}{0.79} & \mstd{96.74}{0.83}
& \mstd{69.94}{0.88} & \mstd{95.61}{0.94} \\

\midrule
\multirow{5}{*}{\textbf{SIDER}}
& GTA
& \mstd{60.83}{0.74} & \mstd{66.47}{1.43}
& \mstd{60.52}{0.78} & \mstd{42.86}{1.74}
& \mstd{60.31}{0.82} & \mstd{39.42}{1.86} \\
& MB
& \mstd{64.73}{0.73} & \mstd{83.28}{1.46}
& \mstd{63.96}{0.87} & \mstd{60.73}{1.92}
& \mstd{63.72}{0.93} & \mstd{56.68}{2.14} \\
& UGBA
& \mstd{59.47}{0.64} & \mstd{86.18}{1.74}
& \mstd{59.16}{0.69} & \mstd{82.91}{1.83}
& \mstd{58.92}{0.74} & \mstd{80.83}{1.97} \\
& DPGBA
& \mstd{62.94}{0.63} & \mstd{49.67}{1.04}
& \mstd{62.41}{0.67} & \mstd{47.92}{1.18}
& \mstd{62.14}{0.72} & \mstd{46.73}{1.24} \\
& \textbf{ChemBack}
& \mstd{60.84}{0.63} & \mstd{99.17}{0.34}
& \mstd{60.32}{0.74} & \mstd{96.84}{0.86}
& \mstd{60.14}{0.81} & \mstd{95.37}{1.04} \\

\midrule
\multirow{5}{*}{\textbf{Tox21}}
& GTA
& \mstd{96.43}{0.43} & \mstd{42.78}{1.29}
& \mstd{96.82}{0.31} & \mstd{28.16}{1.42}
& \mstd{96.94}{0.27} & \mstd{25.74}{1.38} \\
& MB
& \mstd{96.94}{0.34} & \mstd{96.37}{0.63}
& \mstd{96.76}{0.33} & \mstd{88.64}{0.78}
& \mstd{96.82}{0.31} & \mstd{82.37}{0.93} \\
& UGBA
& \mstd{96.53}{0.33} & \mstd{97.06}{0.54}
& \mstd{96.61}{0.34} & \mstd{95.83}{0.58}
& \mstd{96.58}{0.32} & \mstd{93.46}{0.76} \\
& DPGBA
& \mstd{96.56}{0.43} & \mstd{52.76}{0.94}
& \mstd{96.42}{0.44} & \mstd{50.61}{0.91}
& \mstd{96.47}{0.42} & \mstd{48.24}{0.86} \\
& \textbf{ChemBack}
& \mstd{96.34}{0.43} & \mstd{81.84}{0.84}
& \mstd{96.12}{0.46} & \mstd{80.46}{0.96}
& \mstd{96.03}{0.51} & \mstd{78.93}{1.17} \\

\bottomrule
\end{tabular}
\end{table*}

\section{Large-scale and Pretrain-finetune Stress Tests}
\label{app:scale-pretrain}

This section reports additional stress tests beyond the four main MoleculeNet datasets. Our goal is not to claim state-of-the-art molecular property prediction performance. Instead, our evaluation focuses on backdoor-specific metrics, including clean accuracy (CA), attack success rate (ASR), and effective poisoning rate (EPR).
Instead, we ask whether the same chemistry-aware pattern persists when the dataset scale increases and when the victim is fine-tuned from a large pre-trained molecular model.

\subsection{Large-scale PCBA and MUV Benchmarks}
\label{app:large-scale-pcba-muv}

We first evaluate PCBA and MUV as larger-scale MoleculeNet stress tests. PCBA contains 128 bioassay prediction tasks and MUV contains 17 virtual-screening tasks. For scalability, we follow the same targeted binary attack protocol as the main experiments and evaluate a representative target head for each dataset. The purpose of this experiment is to test scale, not to exhaustively report all task heads.

Table~\ref{tab:large-scale} reports ASR before and after applying \textbf{ChemGuard}. The raw pipeline shows that several graph-only baselines can still obtain high apparent ASR when graph edits are treated as automatically admissible. However, once \textbf{ChemGuard} is enforced, their operational ASR decreases because only a small fraction of attempted poisons are admitted. In contrast, \textbf{ChemBack} maintains EPR at $100.00\%$ and has identical raw and \textbf{ChemGuard} ASR, showing that its poisons are chemistry-admissible by construction even at larger scale.

\begin{table*}[t]
\centering
\caption{
Large-scale benchmark stress tests on PCBA and MUV at $\alpha=10\%$.
Raw denotes the abstract-graph evaluation without chemistry-aware admission.
Values are mean$\pm$std over 5 seeds.
}
\label{tab:large-scale}
\footnotesize
\setlength{\tabcolsep}{6pt}
\begin{tabular}{ccccccc}
\toprule
\multirow{2}{*}{\textbf{Dataset}} &
\multirow{2}{*}{\textbf{Attack}} &
\multicolumn{2}{c}{\textbf{Raw}} &
\multicolumn{3}{c}{\textbf{with \textbf{ChemGuard}}} \\
\cmidrule(lr){3-4}
\cmidrule(lr){5-7}
& &
\textbf{CA (\%) $\uparrow$} &
\textbf{ASR (\%) $\uparrow$} &
\textbf{CA (\%) $\uparrow$} &
\textbf{ASR (\%) $\uparrow$} &
\textbf{EPR (\%) $\uparrow$} \\
\midrule

\multirow{5}{*}{\textbf{PCBA}}
& GTA
& \mstd{89.94}{0.41}
& \mstd{98.67}{1.28}
& \mstd{90.28}{0.37}
& \mstd{56.67}{2.18}
& \mstd{14.73}{1.96} \\
& MB
& \mstd{89.86}{0.38}
& \mstd{99.67}{0.74}
& \mstd{90.27}{0.34}
& \mstd{59.67}{2.41}
& \mstd{19.84}{2.27} \\
& UGBA
& \mstd{90.14}{0.36}
& \mstd{99.83}{0.37}
& \mstd{90.46}{0.31}
& \mstd{50.24}{1.93}
& \mstd{9.72}{1.68} \\
& DPGBA
& \mstd{90.13}{0.39}
& \mstd{87.16}{1.84}
& \mstd{90.47}{0.36}
& \mstd{67.24}{2.69}
& \mstd{14.62}{2.13} \\
& \textbf{ChemBack}
& \mstd{90.23}{0.29}
& \bmstd{92.41}{1.37}
& \mstd{90.23}{0.29}
& \bmstd{92.41}{1.37}
& \bmstd{100.00}{0.00} \\

\midrule

\multirow{5}{*}{\textbf{MUV}}
& GTA
& \mstd{99.42}{0.16}
& \mstd{74.13}{2.37}
& \mstd{99.56}{0.12}
& \mstd{34.17}{2.26}
& \mstd{13.79}{1.83} \\
& MB
& \mstd{99.58}{0.13}
& \mstd{96.87}{1.16}
& \mstd{99.66}{0.09}
& \mstd{51.12}{2.47}
& \mstd{20.32}{2.54} \\
& UGBA
& \mstd{99.59}{0.12}
& \mstd{94.26}{1.34}
& \mstd{99.66}{0.11}
& \mstd{20.14}{1.73}
& \mstd{9.47}{1.61} \\
& DPGBA
& \mstd{99.54}{0.14}
& \mstd{82.39}{1.92}
& \mstd{99.66}{0.13}
& \mstd{36.08}{2.18}
& \mstd{14.81}{1.94} \\
& \textbf{ChemBack}
& \mstd{99.43}{0.16}
& \bmstd{93.32}{1.46}
& \mstd{99.43}{0.16}
& \bmstd{93.32}{1.46}
& \bmstd{100.00}{0.00} \\
\bottomrule
\end{tabular}
\end{table*}

\subsection{GROVER Pretrain-finetune Setting}
\label{app:grover-pretrain}

We next evaluate whether \textbf{ChemBack} remains effective when the victim is not trained from scratch. We use GROVER, a large-scale self-supervised molecular graph transformer, as a pretrain-finetune molecular representation model. This setting addresses whether chemistry-aware backdoors persist in modern molecular pipelines that first pretrain on large unlabeled molecular corpora and then fine-tune on downstream property prediction tasks.

Table~\ref{tab:grover} reports results on all four main datasets.
Compared with the from-scratch GCN setting, GROVER generally yields stronger clean utility, which is expected from pretraining. The attack pattern, however, remains the same. Graph-only baselines can achieve high raw ASR when invalid graph edits are treated as admissible, but their operational ASR drops after \textbf{ChemGuard} because EPR remains low. \textbf{ChemBack} maintains EPR at $100.00\%$ and preserves high ASR across all four datasets, showing that the attack is not restricted to from-scratch GCN training.

\begin{table*}[t]
\centering
\caption{
GROVER pretrain-finetune stress tests at $\alpha=10\%$. Raw denotes evaluation without chemistry-aware admission. \textbf{ChemGuard} denotes operational chemistry-aware evaluation. Values are mean$\pm$std over 5 seeds.
}
\label{tab:grover}
\footnotesize
\setlength{\tabcolsep}{6pt}
\begin{tabular}{ccccccc}
\toprule
\multirow{2}{*}{\textbf{Dataset}} &
\multirow{2}{*}{\textbf{Attack}} &
\multicolumn{2}{c}{\textbf{Raw}} &
\multicolumn{3}{c}{\textbf{with \textbf{ChemGuard}}} \\
\cmidrule(lr){3-4}
\cmidrule(lr){5-7}
& &
\textbf{CA (\%) $\uparrow$} &
\textbf{ASR (\%) $\uparrow$} &
\textbf{CA (\%) $\uparrow$} &
\textbf{ASR (\%) $\uparrow$} &
\textbf{EPR (\%) $\uparrow$} \\
\midrule

\multirow{5}{*}{\textbf{BBBP}}
& GTA
& \mstd{91.93}{0.46}
& \mstd{73.41}{1.33}
& \mstd{92.27}{0.41}
& \mstd{54.47}{1.29}
& \mstd{8.17}{12.86} \\
& MB
& \mstd{93.87}{0.42}
& \mstd{72.36}{1.49}
& \mstd{94.36}{0.38}
& \mstd{46.21}{7.83}
& \mstd{11.27}{11.34} \\
& UGBA
& \mstd{93.42}{0.47}
& \mstd{60.91}{1.24}
& \mstd{93.91}{0.45}
& \mstd{47.84}{3.89}
& \mstd{13.19}{13.57} \\
& DPGBA
& \mstd{93.86}{0.74}
& \mstd{69.51}{1.37}
& \mstd{94.23}{0.72}
& \mstd{52.88}{3.37}
& \mstd{10.09}{10.86} \\
& \textbf{ChemBack}
& \mstd{92.13}{0.52}
& \bmstd{79.29}{1.41}
& \mstd{92.13}{0.52}
& \bmstd{79.29}{1.41}
& \bmstd{100.00}{0.00} \\

\midrule

\multirow{5}{*}{\textbf{BACE}}
& GTA
& \mstd{86.78}{0.69}
& \mstd{48.13}{1.27}
& \mstd{87.34}{0.64}
& \mstd{27.43}{2.12}
& \mstd{15.13}{13.24} \\
& MB
& \mstd{84.76}{0.66}
& \mstd{82.74}{1.69}
& \mstd{85.06}{0.61}
& \mstd{13.33}{8.73}
& \mstd{17.46}{10.78} \\
& UGBA
& \mstd{86.38}{0.59}
& \mstd{67.42}{1.38}
& \mstd{86.71}{0.57}
& \mstd{15.36}{3.31}
& \mstd{10.39}{12.16} \\
& DPGBA
& \mstd{81.83}{0.72}
& \mstd{55.83}{1.31}
& \mstd{82.21}{0.68}
& \mstd{16.46}{4.57}
& \mstd{10.91}{13.94} \\
& \textbf{ChemBack}
& \mstd{85.63}{0.73}
& \bmstd{98.81}{0.47}
& \mstd{85.63}{0.73}
& \bmstd{98.81}{0.47}
& \bmstd{100.00}{0.00} \\

\midrule

\multirow{5}{*}{\textbf{SIDER}}
& GTA
& \mstd{63.72}{0.81}
& \mstd{66.83}{1.46}
& \mstd{63.94}{0.74}
& \mstd{50.13}{7.61}
& \mstd{11.13}{11.86} \\
& MB
& \mstd{65.34}{0.76}
& \mstd{83.74}{1.52}
& \mstd{64.86}{1.34}
& \mstd{51.74}{3.91}
& \mstd{7.08}{12.47} \\
& UGBA
& \mstd{63.16}{0.68}
& \mstd{86.52}{1.76}
& \mstd{63.42}{0.93}
& \mstd{53.83}{4.27}
& \mstd{8.57}{10.26} \\
& DPGBA
& \mstd{64.27}{0.69}
& \mstd{50.16}{1.18}
& \mstd{64.13}{0.52}
& \mstd{45.62}{4.31}
& \mstd{5.47}{12.94} \\
& \textbf{ChemBack}
& \mstd{63.84}{0.66}
& \bmstd{98.63}{0.51}
& \mstd{63.84}{0.66}
& \bmstd{98.63}{0.51}
& \bmstd{100.00}{0.00} \\

\midrule

\multirow{5}{*}{\textbf{Tox21}}
& GTA
& \mstd{97.48}{0.29}
& \mstd{44.28}{1.47}
& \mstd{97.61}{0.21}
& \mstd{18.47}{2.34}
& \mstd{24.71}{14.08} \\
& MB
& \mstd{97.36}{0.31}
& \mstd{96.54}{0.67}
& \mstd{97.64}{0.16}
& \mstd{25.18}{3.69}
& \mstd{27.16}{11.23} \\
& UGBA
& \mstd{97.23}{0.34}
& \mstd{97.28}{0.59}
& \mstd{97.58}{0.31}
& \mstd{23.62}{2.58}
& \mstd{25.37}{12.84} \\
& DPGBA
& \mstd{97.31}{0.42}
& \mstd{53.16}{0.97}
& \mstd{97.57}{0.24}
& \mstd{19.28}{2.74}
& \mstd{26.84}{10.58} \\
& \textbf{ChemBack}
& \mstd{97.06}{0.41}
& \bmstd{84.76}{0.89}
& \mstd{97.06}{0.41}
& \bmstd{84.76}{0.89}
& \bmstd{100.00}{0.00} \\
\bottomrule
\end{tabular}
\end{table*}

Overall, these stress tests support two conclusions. First, the low EPR of graph-only baselines is not an artifact of the four small datasets used in the main text; the same admission-stage failure appears on larger molecular benchmarks. Second, \textbf{ChemBack}'s effectiveness is not tied to from-scratch GCN training. Because its trigger is chemically admissible and target-aligned at the molecular-structure level, the poisoned signal remains learnable under both larger datasets and pretrain-finetune molecular representation models.
\section{Task-wise Results on Multi-task Benchmarks}
\label{app:taskwise-asr-chemback}

To complement macro-averaged results in the main paper, we report task-wise ASR for \textbf{ChemBack} on SIDER and Tox21. All settings follow the main experimental protocol, including the same poison rate, victim training configuration, and \textbf{ChemGuard} enforcement at training and evaluation time.

\begin{figure*}[ht]
  \centering
  \begin{subfigure}[t]{0.77\linewidth}
    \centering
    \includegraphics[width=\linewidth]{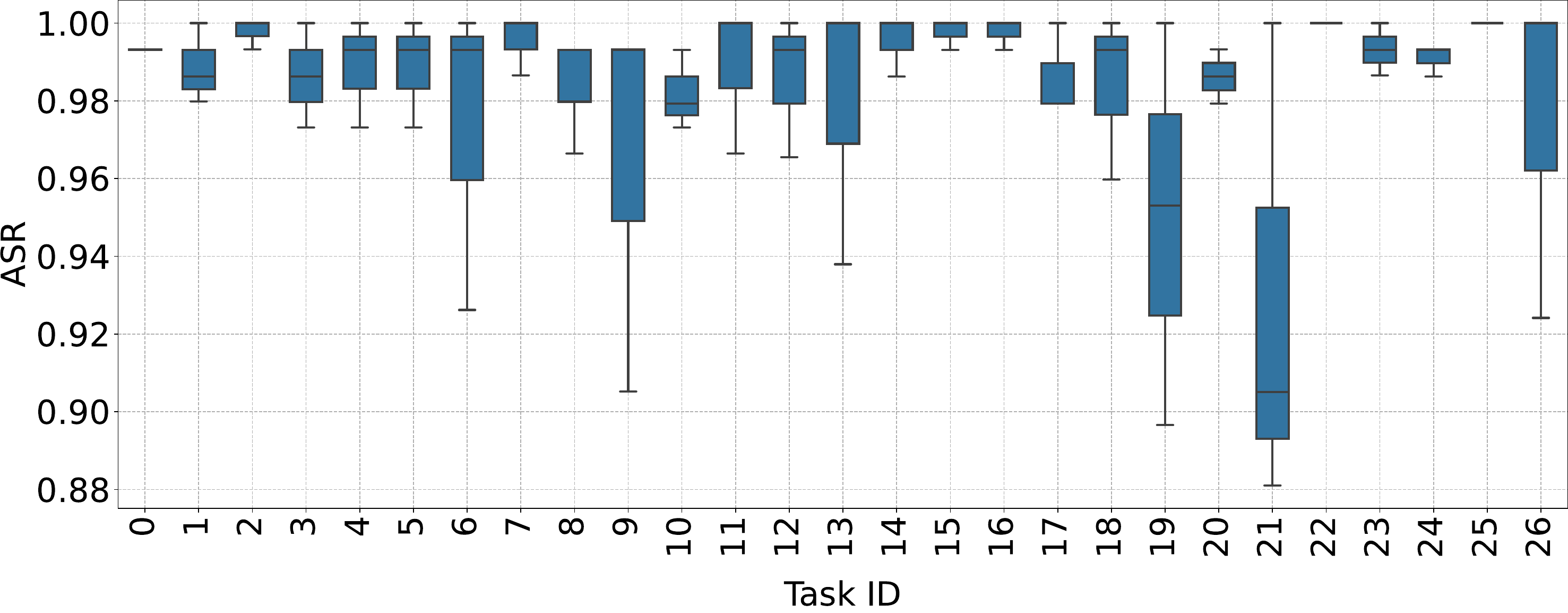}
    \caption{SIDER.}
  \end{subfigure}\hfill
  \begin{subfigure}[t]{0.6\linewidth}
    \centering
    \includegraphics[width=\linewidth]{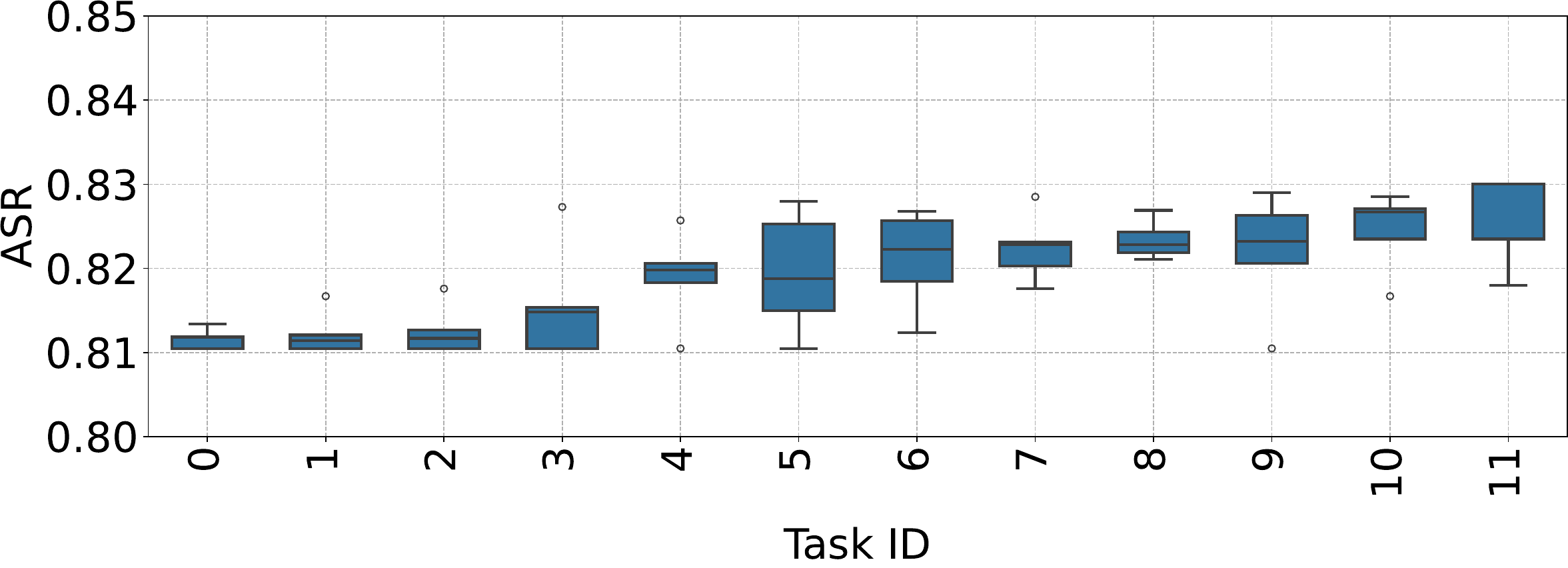}
    \caption{Tox21.}
  \end{subfigure}
  \caption{Task-wise ASR for \textbf{ChemBack} under \textbf{ChemGuard} on multi-task benchmarks.
  Each box summarizes the distribution of ASR across seeds for each task.
  While ASR varies across endpoints, \textbf{ChemBack} remains consistently effective across the evaluated task panel.}
  \label{fig:chemback-taskwise-asr}
\end{figure*}

\section{Theoretical Analysis}
\label{app:theory}

This appendix provides the theoretical analysis of the chemistry-aware evaluation protocol and the design choices in \textbf{ChemBack}. Throughout this appendix, let
\[
\mathrm{Adm}(\tilde{s},\tilde{G})\in\{0,1\}
\]
denote an admission rule. In the raw abstract-graph setting, $\mathrm{Adm}(\tilde{s},\tilde{G})\equiv 1$. Under chemistry-aware evaluation, $\mathrm{Adm}(\tilde{s},\tilde{G})=\mathrm{ChemGuard}_{\mathcal T}(\tilde{s},\tilde{G})$. Let $\tau=(\tau_s,\tau_G)$ denote a trigger transformation, where $\tau_s$ modifies the molecular string and $\tau_G$ denotes the submitted molecular graph after the same modification. We write
\[
\tau(\tilde{s},\tilde{G}) = (\tau_s(\tilde{s}),\tau_G(\tilde{G})).
\]

\subsection{Admission-aware Attack Success}
\label{app:theory-admission-asr}

The first observation formalizes why abstract-graph ASR can overestimate operational risk in molecular pipelines.
A triggered graph can only activate the deployed model if the triggered molecular record is admitted by preprocessing.

\begin{proposition}[Operational ASR is admission-gated]
\label{prop:admission-asr}
Let
\[
\mathcal C =
\{(\tilde{s},\tilde{G},y): \mathrm{A}(\tilde{s},\tilde{G})=1,\ y\neq y_t\}
\]
be the admitted non-target test population. Define the test-time trigger admission rate
\[
q_{\mathrm{test}}
=
\Pr\!\left[
\mathrm{Adm}(\tau_s(\tilde{s}),\tau_G(\tilde{G}))=1
\mid (\tilde{s},\tilde{G},y)\in\mathcal C
\right].
\]
Define the conditional attack success among triggered molecules that pass admission as
\[
r_{\mathrm{cond}}
=
\Pr\!\left[
\arg\max_c f_\theta(\phi_{\mathcal T}(\tau_s(s)))_c=y_t
\mid
\mathrm{Adm}(\tau_s(\tilde{s}),\tau_G(\tilde{G}))=1,\ (\tilde{s},\tilde{G},y)\in\mathcal C
\right],
\]
with the convention that $r_{\mathrm{cond}}=0$ when $q_{\mathrm{test}}=0$. Then the operational ASR satisfies
\[
\mathrm{ASR}
=
q_{\mathrm{test}}\cdot r_{\mathrm{cond}}
\le q_{\mathrm{test}}.
\]
\end{proposition}

\begin{proof}
By the operational ASR definition,
\[
\mathrm{ASR}
=
\Pr\!\left[
\mathrm{Adm}(\tau_s(\tilde{s}),\tau_G(\tilde{G}))=1
\ \wedge\
\arg\max_c f_\theta(\phi_{\mathcal T}(\tau_s(\tilde{s})))_c=y_t
\mid
(\tilde{s},\tilde{G},y)\in\mathcal C
\right].
\]
Applying the product rule gives
\[
\begin{aligned}
\mathrm{ASR}
&=
\Pr\!\left[
\mathrm{Adm}(\tau_s(\tilde{s}),\tau_G(\tilde{G}))=1
\mid
(\tilde{s},\tilde{G},y)\in\mathcal C
\right]
\\
&\times\Pr\!\left[
\arg\max_c f_\theta(\phi_{\mathcal T}(\tau_s(s)))_c=y_t
\mid
\mathrm{Adm}(\tau_s(\tilde{s}),\tau_G(\tilde{G}))=1,\ (\tilde{s},\tilde{G},y)\in\mathcal C
\right].
\end{aligned}
\]
The first factor is $q_{\mathrm{test}}$ and the second factor is $r_{\mathrm{cond}}$.

Since $r_{\mathrm{cond}}\in[0,1]$, we obtain $\mathrm{ASR}\le q_{\mathrm{test}}$.
\end{proof}

This proposition explains the main evaluation gap. Even if a graph trigger appears effective in raw evaluation, its operational ASR is upper-bounded by the probability that the triggered molecular record survives chemistry-aware admission. Thus, invalid or graph-string inconsistent test-time triggers must be counted as attack failures rather than successful triggered inputs.

A similar accounting applies to training-time poisoning. Suppose the attacker attempts a nominal poison rate $\alpha$ before admission.
Let
\[
q_p
=
\Pr[\mathrm{Adm}(\tilde{s}_p,\tilde{G}_p)=1\mid (\tilde{s}_p,\tilde{G}_p,y_t)\in D_{\mathrm{poison}}^{\mathrm{att}}],
\]
and let
\[
q_c
=
\Pr[\mathrm{Adm}(\tilde{s},\tilde{G})=1\mid (\tilde{s},\tilde{G},y)\in D_{\mathrm{clean}}].
\]
Then the admitted poison fraction in the effective training distribution is
\begin{equation}
\alpha_{\mathrm{eff}}
=
\frac{\alpha q_p}
{(1-\alpha)q_c+\alpha q_p}.
\label{eq:alpha-eff-theory}
\end{equation}
When clean molecular records are mostly valid, $q_{\mathrm{clean}}\approx 1$, and
\[
\alpha_{\mathrm{eff}}
\approx
\frac{\alpha q_p}{1-\alpha+\alpha q_p}.
\]
Thus, a low poison admission rate directly reduces the amount of poisoned signal that can reach training.

\begin{proof}
Before admission, the training distribution is a mixture with clean mass $(1-\alpha)$ and attempted poison mass $\alpha$. After admission, clean records contribute mass $(1-\alpha)q_c$, while attempted poisons contribute mass $\alpha q_p$. Normalizing the admitted poison mass by the total admitted mass gives Equation~\eqref{eq:alpha-eff-theory}.
\end{proof}

\subsection{\textbf{ChemBack} Admissibility by Construction}
\label{app:theory-admissibility}

The second result formalizes why \textbf{ChemBack} reports EPR $=100\%$ for submitted poisons. This does not mean that every internal motif-anchor attempt succeeds. Rather, invalid internal attempts are rejected during search; only admitted poisoned records are submitted to the poisoned training set.

\begin{proposition}[\textbf{ChemBack} submitted poisons pass \textbf{ChemGuard}]
\label{prop:admissibility-by-construction}
Assume the attachment operator $\mathrm{Att}_{\mathcal T}(G,m_k,i,j)$ returns a molecular record $(\tilde{s}_p,\tilde{G}_p)$ only if
\[
\phi_{\mathcal T}(\tilde{s}_p)\neq\varnothing
\quad\text{and}\quad
\mathrm{Topo}(\phi_{\mathcal T}(\tilde{s}_p))=\mathrm{Topo}(\tilde{G}_p).
\]
Otherwise, the operator returns $\bot$.
If the final submitted poisoned set $D_{\mathrm{poison}}^{\mathrm{sub}}$ contains only records returned by $\mathrm{Att}_{\mathcal T}$, then
\[
\mathrm{ChemGuard}_{\mathcal T}(\tilde{s}_p,\tilde{G}_p)=1
\quad
\text{for all }(\tilde{s}_p,\tilde{G}_p,y_t)\in D_{\mathrm{poison}}^{\mathrm{sub}},
\]
and therefore
\[
\mathrm{EPR}=1.
\]
\end{proposition}

\begin{proof}
By the definition of \textbf{ChemGuard},
\[
\mathrm{ChemGuard}_{\mathcal T}(\tilde{s}_p,\tilde{G}_p)
=
\mathbf{1}[\phi_{\mathcal T}(\tilde{s}_p)\neq\varnothing]
\cdot
\mathbf{1}[
\mathrm{Topo}(\phi_{\mathcal T}(\tilde{s}_p))=\mathrm{Topo}(\tilde{G}_p)
].
\]
For every submitted poison generated by $\mathrm{Att}_{\mathcal T}$, both indicators are one by assumption. Thus, $\mathrm{ChemGuard}_{\mathcal T}(\tilde{s}_p,\tilde{G}_p)=1$ for every submitted poisoned record. The EPR is the average of this indicator over submitted attempted poisons, so $\mathrm{EPR}=1$.
\end{proof}

This proposition is intentionally about submitted poisoned records, not about all internal candidate proposals. A search procedure may explore many invalid motif-anchor actions, but \textbf{ChemBack} filters them before submission. This distinction is important because EPR measures the poison signal that can enter training, not the number of invalid candidates explored during black-box trigger search.

\subsection{Clean-risk Preservation under Rare Triggers}
\label{app:theory-clean-risk}

Backdoor attacks aim to preserve clean utility while changing behavior only when the trigger is present. The following simple bound explains why rare molecular triggers can preserve clean accuracy, provided that the trigger predicate rarely appears naturally in clean non-target data.

Let $z(G)\in\{0,1\}$ denote a trigger predicate, where $z(G)=1$ if molecule $G$ contains the trigger substructure. Let $f_0$ be a clean classifier. Define the idealized backdoored classifier
\[
f_{\mathrm{bd}}(G)
=
\begin{cases}
y_t, & z(G)=1,\\
f_0(G), & z(G)=0.
\end{cases}
\]

\begin{proposition}[Rare triggers preserve clean risk]
\label{prop:rare-trigger-risk}
Let $R_{\mathrm{clean}}(f)=\Pr_{(G,y)\sim P_{\mathrm{clean}}}[f(G)\neq y]$.
Then
\[
R_{\mathrm{clean}}(f_{\mathrm{bd}})
\le
R_{\mathrm{clean}}(f_0)
+
\Pr_{(G,y)\sim P_{\mathrm{clean}}}[z(G)=1,\ y\neq y_t].
\]
\end{proposition}

\begin{proof}
For any clean sample $(G,y)$, if $z(G)=0$, then $f_{\mathrm{bd}}(G)=f_0(G)$ and the two classifiers have identical error. If $z(G)=1$ and $y=y_t$, then $f_{\mathrm{bd}}(G)=y_t=y$, so the backdoored classifier is correct on that sample. The only additional clean error caused by the trigger rule can occur when $z(G)=1$ and $y\neq y_t$.
Therefore, pointwise,
\[
\mathbf{1}[f_{\mathrm{bd}}(G)\neq y]
\le
\mathbf{1}[f_0(G)\neq y]
+
\mathbf{1}[z(G)=1,\ y\neq y_t].
\]
Taking expectation over $P_{\mathrm{clean}}$ proves the claim.
\end{proof}

This result justifies using rare motifs to reduce accidental clean trigger occurrence. However, rarity alone is not a stealth guarantee.
A rare motif may still be structurally or representationally outlying.
This is why \textbf{ChemBack} additionally selects among admitted candidates using Tanimoto similarity to clean target-class molecules.

\subsection{Tanimoto Similarity and Representation Proximity}
\label{app:theory-tanimoto-md2}

\textbf{ChemBack} does not use a clean victim model during trigger selection. It ranks admitted candidates by fingerprint-based Tanimoto similarity to clean target-class molecules. This subsection explains why such a chemistry-native structural signal can align with the post-hoc MD$^2$ diagnostic, without making MD$^2$ part of the attack.

Let $\mathrm{FP}(G)\in\{0,1\}^d$ be a molecular fingerprint and define the Tanimoto distance
\[
d_{\mathrm{Tan}}(G,G')
=
1-\mathrm{Tan}(\mathrm{FP}(G),\mathrm{FP}(G')).
\]
Let $h_\theta(G)\in\mathbb{R}^p$ denote the penultimate representation of a clean reference model used only for diagnostics.
For the clean target class $y_t$, let $\mu_t$ and $\Sigma_t$ be the empirical mean and covariance of clean target embeddings.
We use a ridge-regularized covariance
\[
\widetilde{\Sigma}_t=\Sigma_t+\rho I,
\quad \rho>0,
\]
and define
\[
\mathrm{MD}_{t}(G)
=
\left\|
\widetilde{\Sigma}_t^{-1/2}
\left(h_\theta(G)-\mu_t\right)
\right\|_2.
\]
Thus $\mathrm{MD}_{t}^{2}(G)$ is the squared Mahalanobis distance to the clean target class.

\begin{assumption}[Local fingerprint-representation smoothness]
\label{assump:local-smoothness}
For the molecular neighborhood considered by \textbf{ChemBack}, there exists $L>0$ such that for any admitted molecules $G$ and $G'$,
\[
\|h_\theta(G)-h_\theta(G')\|_2
\le
L\, d_{\mathrm{Tan}}(G,G').
\]
\end{assumption}

This assumption does not grant the attacker access to $h_\theta$.
It is only a post-hoc regularity condition used to explain why structural similarity can align with learned representation proximity.

\begin{proposition}[Tanimoto target similarity implies bounded MD under local smoothness]
\label{prop:tanimoto-md2-bound}
Suppose Assumption~\ref{assump:local-smoothness} holds.
Let $G_p$ be a \textbf{ChemGuard}-admissible poison.
Assume there exists a clean target-class molecule $G_z$ such that
\[
d_{\mathrm{Tan}}(G_p,G_z)\le \epsilon
\quad\text{and}\quad
\mathrm{MD}_{t}^{2}(G_z)\le r.
\]
Let $\lambda_{\min}(\widetilde{\Sigma}_t)$ be the smallest eigenvalue of $\widetilde{\Sigma}_t$.
Then
\[
\mathrm{MD}_{t}(G_p)
\le
\sqrt{r}
+
\frac{L}{\sqrt{\lambda_{\min}(\widetilde{\Sigma}_t)}}\epsilon,
\]
and consequently,
\[
\mathrm{MD}_{t}^{2}(G_p)
\le
\left(
\sqrt{r}
+
\frac{L}{\sqrt{\lambda_{\min}(\widetilde{\Sigma}_t)}}\epsilon
\right)^2.
\]
\end{proposition}

\begin{proof}
By the triangle inequality in the Mahalanobis norm,
\[
\mathrm{MD}_{t}(G_p)
=
\left\|
\widetilde{\Sigma}_t^{-1/2}
(h_\theta(G_p)-\mu_t)
\right\|_2
\]
\[
\le
\left\|
\widetilde{\Sigma}_t^{-1/2}
(h_\theta(G_z)-\mu_t)
\right\|_2
+
\left\|
\widetilde{\Sigma}_t^{-1/2}
(h_\theta(G_p)-h_\theta(G_z))
\right\|_2.
\]
The first term is $\mathrm{MD}_{t}(G_z)\le\sqrt{r}$.
For the second term,
\[
\left\|
\widetilde{\Sigma}_t^{-1/2}
(h_\theta(G_p)-h_\theta(G_z))
\right\|_2
\le
\|\widetilde{\Sigma}_t^{-1/2}\|_2
\cdot
\|h_\theta(G_p)-h_\theta(G_z)\|_2.
\]
Since
\[
\|\widetilde{\Sigma}_t^{-1/2}\|_2
=
\frac{1}{\sqrt{\lambda_{\min}(\widetilde{\Sigma}_t)}},
\]
and by Assumption~\ref{assump:local-smoothness},
\[
\|h_\theta(G_p)-h_\theta(G_z)\|_2
\le
L\, d_{\mathrm{Tan}}(G_p,G_z)
\le
L\epsilon,
\]
we obtain
\[
\mathrm{MD}_{t}(G_p)
\le
\sqrt{r}
+
\frac{L}{\sqrt{\lambda_{\min}(\widetilde{\Sigma}_t)}}\epsilon.
\]
Squaring both sides gives the bound on $\mathrm{MD}_{t}^{2}(G_p)$.
\end{proof}

This proposition provides a formal explanation for the post-hoc diagnostic in the main text.
If Tanimoto similarity is high, then $\epsilon$ is small.
Under local smoothness, a structurally target-similar poison is also close to the target region in the clean model's representation space.
Importantly, the attack does not need to know $h_\theta$, $\mu_t$, or $\Sigma_t$.
These quantities are used only by the evaluator after trigger generation.

\section{Technical Limitations}
\label{app:limitations}

This section expands the limitations briefly discussed in the main conclusion.

\textbf{\textbf{ChemGuard} is a first-layer admission rule, not a complete defense.} \textbf{ChemGuard} formalizes molecular sanitization and graph-string consistency checks that are already implicit in many molecular learning pipelines. It does not certify robustness against all chemistry-valid attacks. It also does not cover all domain-specific screening procedures used in medicinal chemistry or materials discovery.

\textbf{Toolkit sanitization is not full chemical realism.}
Passing RDKit, Indigo, or Open Babel sanitization does not imply synthetic accessibility, biological relevance, assay compatibility, or expert approval. These toolkits enforce important but limited notions of molecular validity. Future work should incorporate deeper chemistry constraints, including synthetic accessibility, retrosynthetic feasibility, reaction validity, PAINS filters, medicinal-chemistry rules, and expert review.

\textbf{\textbf{ChemBack} uses a controlled single-step attachment space.} \textbf{ChemBack} currently uses motif attachment through a single chemically valid bond. This controlled design isolates the effect of chemistry-aware admission and makes the attack interpretable. However, richer multi-step or reaction-like transformations may produce broader classes of valid molecular backdoors. Studying such transformations is an important future direction.

\textbf{Tanimoto similarity is a model-free structural signal, not a universal stealth guarantee.} \textbf{ChemBack} uses Tanimoto similarity to select poisons structurally close to the clean target class. Post-hoc diagnostics show alignment with clean-model MD$^2$ and chemistry-facing descriptors, but these diagnostics do not guarantee human-level stealth or robustness against every possible structural detector. The stealth claims in this work are therefore limited to the evaluated diagnostics.

\textbf{Benchmarks are proxies for deployed molecular pipelines.}
MoleculeNet benchmarks provide standardized evaluation settings, but deployed molecular ML systems may include proprietary validation rules, temporal data shifts, assay-specific curation, human review, or additional domain filters. Our results show that chemistry-aware admission materially changes molecular backdoor evaluation, but further work is needed to study more complex deployment pipelines.

\end{document}